\documentclass{article}

\usepackage[numbers]{natbib}

 \usepackage[main, final]{neurips_2026}

\usepackage[utf8]{inputenc} %
\usepackage[T1]{fontenc}    %
\usepackage{hyperref}       %
\usepackage{url}            %
\usepackage{booktabs}       %
\usepackage{amsfonts}       %
\usepackage{nicefrac}       %
\usepackage{microtype}      %
\usepackage{xcolor}         %
\usepackage{booktabs}
\usepackage[table]{xcolor}

\usepackage{mathtools, amssymb, amsthm}
\usepackage{bbm}
\usepackage{subcaption}
\usepackage{pgf}
\usepackage{wrapfig}
\usepackage{enumitem}
\usepackage{thm-restate}

\theoremstyle{plain}
\newtheorem{theorem}{Theorem}

\theoremstyle{definition}
\newtheorem{definition}[theorem]{Definition}

\newtheorem{assumption}{Assumption}

\renewcommand{\eqref}[1]{equation~(\ref{#1})}

\newcommand{\z}{\zeta}
\newcommand{\one}{\mathbbm{1}}

\newcommand{\abs}[1]{\lvert #1 \rvert}

\newcommand{\magwt}{\mathbf{w}}

\newcommand\restr[2]{{%
  \left.\kern-\nulldelimiterspace %
  #1 %
  \vphantom{\big|} %
  \right|_{#2} %
  }}

\title{Generalization in Neural Networks \\ Through the Lens of Magnitude Potential}

\author{%
Sahel Torkamani\\
School of Informatics\\
University of Edinburgh\\
\texttt{sahel.torkamani@ed.ac.uk} \\
\And 
Henry Gouk \\
School of Informatics\\
University of Edinburgh\\
\texttt{henry.gouk@ed.ac.uk} \\
\And 
Rik Sarkar \\
School of Informatics\\
University of Edinburgh\\
\texttt{rsarkar@inf.ed.ac.uk} \\
}

\begin{document}

\maketitle

\begin{abstract}

Explaining generalization and training dynamics in neural networks remains a challenge, and various approaches have been developed to study different aspects of these phenomena. In this paper, we introduce the idea of {\em magnitude potential} -- a quantity based on the theory of metric magnitude -- that reflects how well an arbitrary point is represented by a given set. We find that this basic quantity can be applied to examine various features in neural generalization. The ratio between the magnitude potential with respect to a class and with respect to the entire data, computed at the logit layer, is informative of the representation of the point. In experiments, these ratios for individual training points are found to be correlated with the Feldman memorization scores. Magnitude potential ratios aggregated across points detect structural changes in the decision boundaries and provide a geometric indicator of grokking in modular arithmetic. Although the magnitude potential ratio and neural collapse are both closely associated with intra-class and inter-class geometric structure, the magnitude potential ratio remains informative even when neural collapse is explicitly suppressed.

\end{abstract}

\section{Introduction}

Neural networks have become popular due to their excellent generalization ability, but the mechanisms behind this generalization remain a major area of study. The behavior of overparameterized neural networks is difficult to analyze using traditional statistical learning theory~\cite{belkin2019reconciling}. These networks regularly memorize certain volumes of data, and indeed, memorization has been argued to be an essential component for generalization in these networks~\cite{feldman2020does}. With continued training, large neural networks are capable of memorizing entire training datasets~\cite{yu2024generalizablity, bonnaire2025why, gu2025on}.  

More recently, it has been found that in domains where generalization to test datasets is difficult, prolonged training past memorization can cause the sudden emergence of generalization. This phenomenon, called {\em grokking}, was first shown for algorithmically generated modular arithmetic datasets~\cite {power2022grokking} and then observed in more general data~\cite{liu2023omnigrok, humayun2024deep}. Various approaches based on network weight decay~\cite{liu2023omnigrok}, the geometry of loss landscapes~\cite{han2025flatness}, statistical physics~\cite{zhang2026is}, and other approaches have been used to develop our understanding of neural network training behavior.  

Meanwhile, metric magnitude~\cite{leinster2013magnitude} has been shown to provide insights into closely related work attempting to understand different aspects of network training. Examples include generalization~\cite{andreeva2023metric,andreeva2024topological}, diversity of latent representations~\cite{limbeck2024metric}, and quantifying distances between datasets~\cite{torkamani2026magnitude}. In this paper, we show that an extension of the metric magnitude called \emph{Magnitude Potential} can be used to study and predict a range of phenomena in training dynamics. 
For a given dataset $X$ and a test data point $z$, the magnitude potential, $\rho^t(z, X)$, can be seen as the degree to which $z$ is represented within $X$ at scale $t$.
It is a continuous Lipschitz function applicable to both training and unseen points.
The ratio of the potential at $z$ relative to points in its own class, to its potential relative to the entire dataset, provides a per-sample geometric signal that is closely related to memorization, and captures training phase transitions when aggregated across 
the dataset.

\paragraph{Our contributions.} 
Metric Magnitude~\cite{leinster2008euler} is a measure of the size or effective number of points in a metric space. For an embedded finite point set $X\in \mathbb{R}^D$ such as is often used in machine learning, it assigns a weight $\magwt(x)$ to every vertex $x\in X$, roughly representing the distinctness of $x$ within $X$, and $\mathrm{Mag}(X)$ is the sum of these weights. The weights are computed based on a  similarity measure $\zeta(x_i, x_j)=\exp(-d(x_i, x_j))$, and satisfy that at any $x_i$, the sum $\sum_{x_j}\zeta(x_i, x_j)\magwt(x_j) = 1$. 

{\em Magnitude potential} is an extension of this idea to all points $z\in \mathbb{R}^D$, so that the potential at $z$ is $\rho(z, X) = \sum_{x_j}\zeta(z, x_j)\magwt(x_j)$. It can be seen as the influence of all points of $X$ that is felt at $z$. The following intuition for this potential is given by Meckes~\cite{meckes2015magnitude}. Suppose we have a type of ``electric charge'' such that a unit charge at $x$ creates a potential of $\zeta(z, x)$ at $z$. The points $X$ are connected to a ground from which charge can flow; they each have a charge of $\magwt_x$ such that the potential at $x$ is exactly $1$. Then the potential at any $z\in \mathbb{R}^D$ is given by $\rho(z, X)$.

We observe that in the machine learning context, this function can be used as a degree of coverage or representation of $z$ by $X$. Unlike simple coverage with unit balls, we prove in Section~\ref{sec:mp} that the magnitude potential describes a {\em soft coverage} score that decays gracefully with distance from $X$. At the same time, compared to kernel density estimation, the magnitude potential is less sensitive to high-density clusters and can incorporate low-density groups and isolated points in a meaningful way to compute representation. This idea is shown in Figure~\ref{fig:mp_vs_kde}. The kernel density of a large cluster can overwhelm a sparser cluster. But magnitude potential allocates larger class scores to sparser clusters and even isolated points. This behavior is similar to outputs of neural networks that can fit to outliers and isolated points. We prove that the magnitude potential is Lipschitz continuous and theoretically well-behaved in limiting regimes: its variance vanishes under within-class variance collapse, and it recovers nearest-neighbor structure when data points scatter.

\begin{figure}[t]
     \centering
     \begin{subfigure}[b]{0.24\textwidth}
         \centering
         \includegraphics[width=\textwidth]{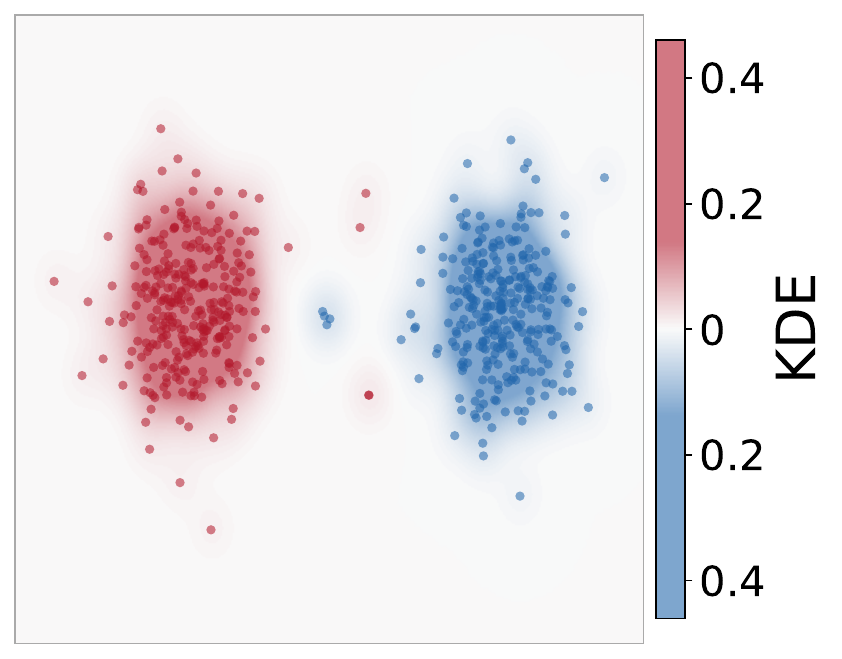}
         \caption{KDE within-class}
         \label{fig:kde_within}
     \end{subfigure}
     \hfill
     \begin{subfigure}[b]{0.24\textwidth}
         \centering
         \includegraphics[width=\textwidth]{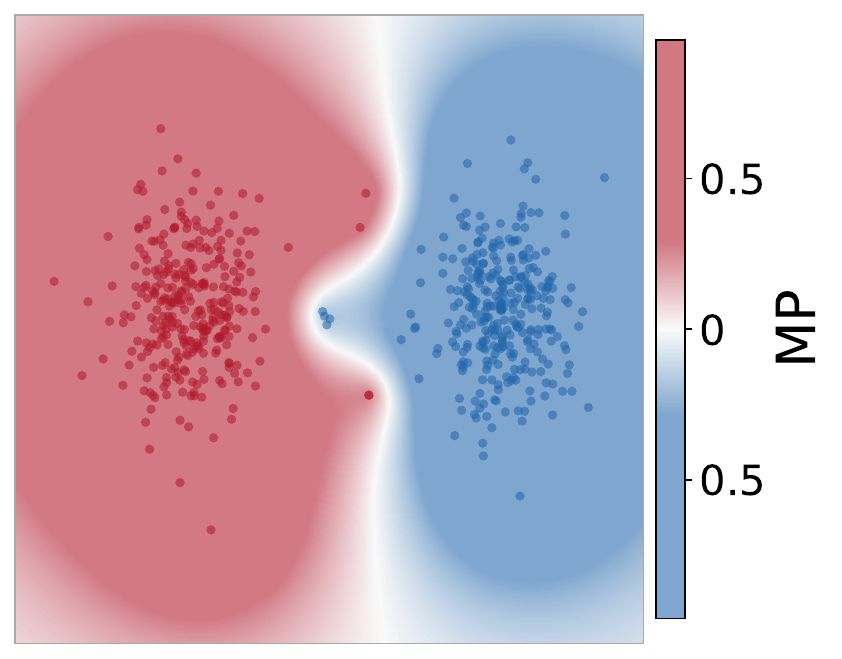}
         \caption{MP ($t = 1.0$)}
         \label{fig:mp_within}
     \end{subfigure}
     \hfill
     \begin{subfigure}[b]{0.24\textwidth}
         \centering
         \includegraphics[width=\textwidth]{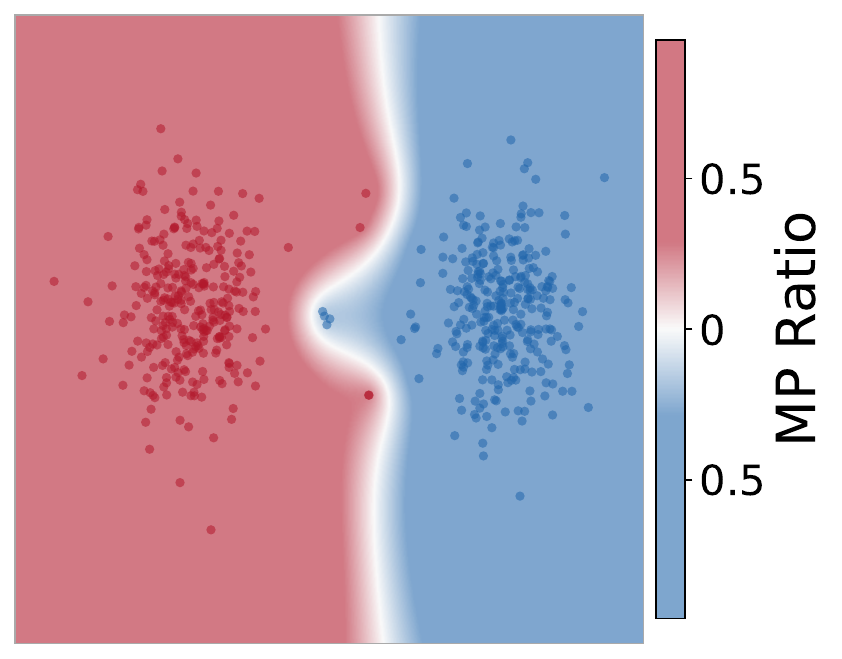}
         \caption{MP ratio ($t = 1.0$)}
         \label{fig:mpratio}
     \end{subfigure}
     \hfill
     \begin{subfigure}[b]{0.24\textwidth}
         \centering
         \includegraphics[width=\textwidth]{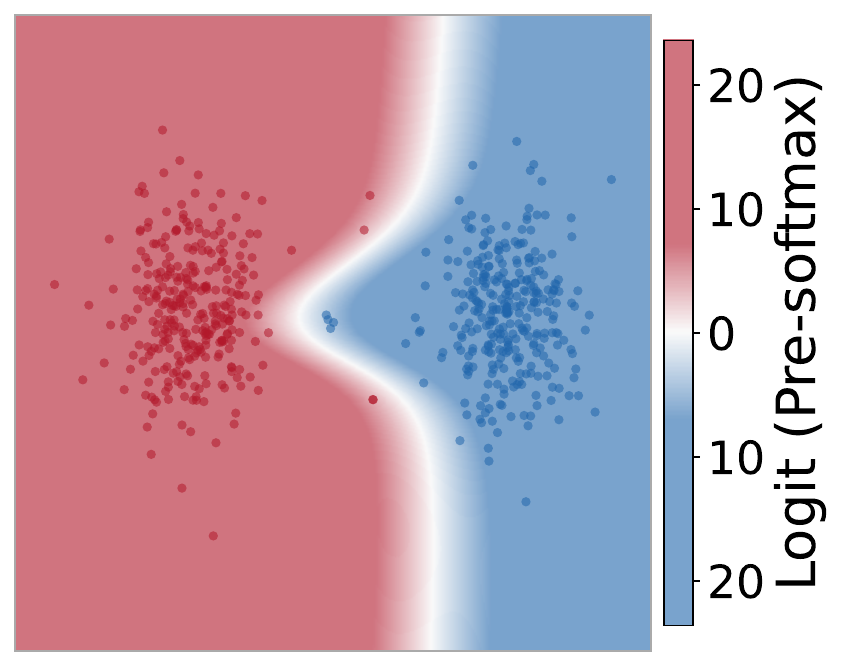}
         \caption{Logit (Pre-softmax)}
         \label{fig:logit}
     \end{subfigure}
     \caption{
     Comparison of KDE and MP. Each class has a cluster and a few sparse points at the boundary.
     {\bf (a)}: KDE is computed for each class, which fails to capture the sparse boundary points as it is dominated by the dense cluster.
     {\bf (b)}: MP is computed for each class, and the largest value is shown, which has a higher score at the outlier points, and shows a decision boundary similar to a neural network in (d). 
     {\bf (c)}: MP ratio, which inherits this sensitivity and bends around the sparse boundary samples.
     {\bf (d)}: the pre-softmax logit of a high-accuracy model, showing similar decision boundaries.
     }
     \label{fig:mp_vs_kde}
\end{figure}

\begin{figure}[t]
    \centering
    \includegraphics[width=\linewidth, trim=20pt 15pt 58pt 58pt, clip]{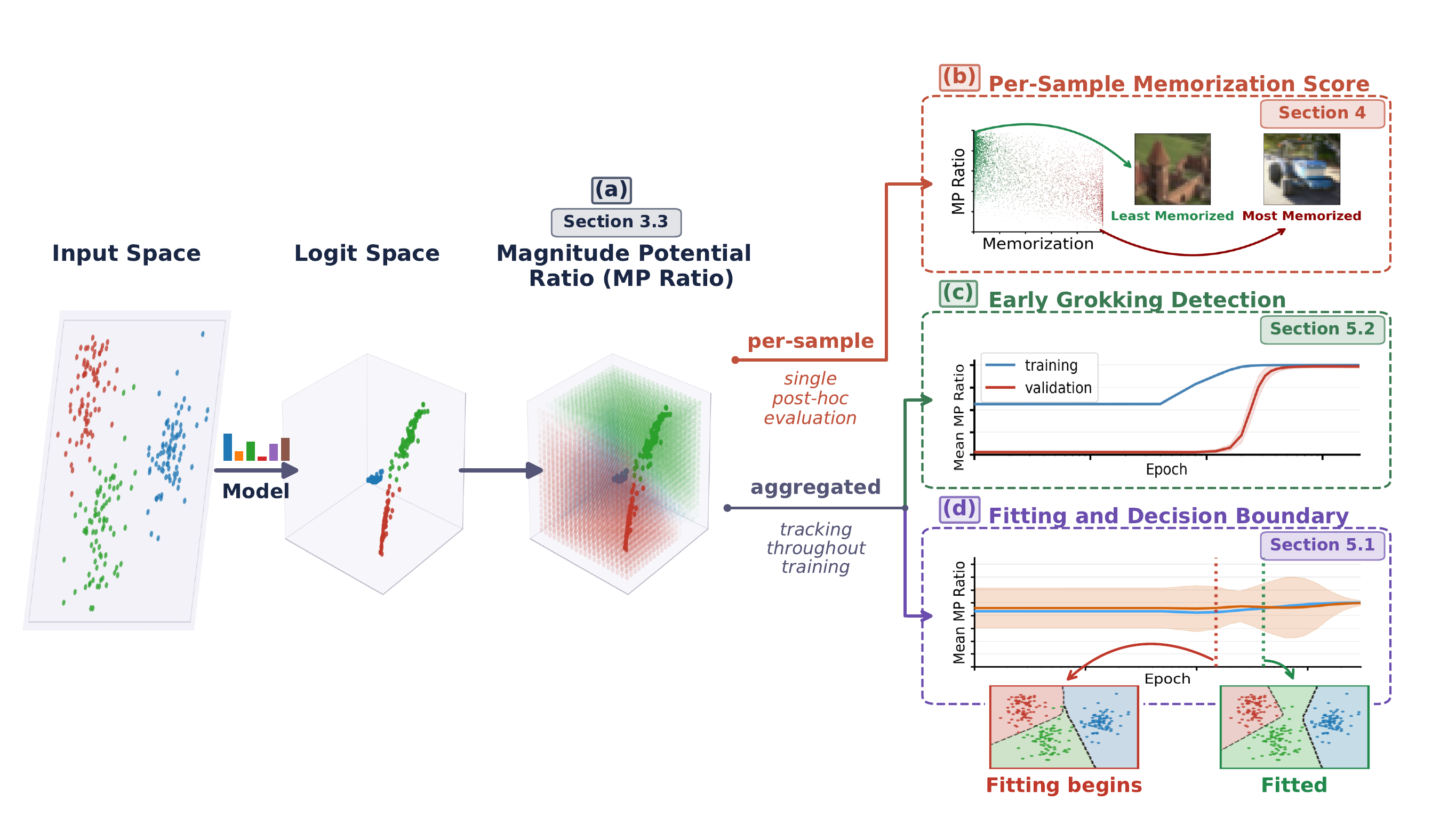}
    \caption{magnitude potential ratio in training dynamics.
  \textbf{(a)}~MP ratio is computed in the logit-layer output space.
  \textbf{(b)}~MP ratio correlates (negatively) with memorization score: lower MP ratio suggests memorization.
  \textbf{(c)}~ Training MP ratio mean indicates grokking prior to generalization being achieved.
  \textbf{(d)}~The validation MP ratio mean and standard deviation track decision boundary changes as the model fits.
}
    \label{fig:concept}
\end{figure}

Next, we describe the {\em Magnitude Potential Ratio} as a ratio between the magnitude potential of $z$ with respect to its own class, and the magnitude potential of $z$ with respect to the entire dataset. Section~\ref{sec:mp} 
shows that points that are heavily memorized tend to be ones with a low magnitude potential ratio. The memorization of points is computed via the memorization score defined in~\cite{feldman2020neural}. 

In our analysis of training dynamics, the magnitude potential ratio is computed on the space of logit values of training and validation data points (Figure~\ref{fig:concept}). In Section~\ref{sec:training_dynamics}, the magnitude potential ratio is shown to undergo changes in sync with training dynamics. In a simple 2D dataset, the standard deviation of the magnitude potential is seen to first drop at points where the decision boundaries transition from linear separators to class-cluster specific boundaries and then rise again when the model starts to overfit (Section~\ref{sec:fitting}). 

In grokking on modular arithmetic data, the increase in the mean of the training MP ratio indicates a geometric signal consistent with the other leading indicators of model generalization, preceding the validation accuracy transition.
At close to 100\% validation accuracy, the standard deviation of MP ratio for both training and validation sets decreases, and the mean value stabilizes (Section~\ref{sec:grokking}). 
These results are seen to hold independent of neural collapse (Section~\ref{sec:robustness_nc}).
Our code is publicly available at \url{https://github.com/saheltorkamani/Magnitude_Potential}.

\section{Preliminaries}\label{sec:prelim}

Magnitude was originally defined as an Euler characteristic for enriched categories~\cite{leinster2008euler}. For a finite metric $(X,d)$, it is defined~\cite{leinster2013magnitude} as follows: 

\begin{definition}[Metric Magnitude]\label{def:finite_magnitude}
A weighting of a metric $(X,d)$ is a function $\magwt: X \to \mathbb{R}$ satisfying $\sum_{j \in X} \zeta(x_i,x_j) \, \magwt(x_j) = 1$, with $\zeta(z_1,z_2):=\exp(-d(z_1,z_2))$,  for every $x_i \in X$,
where $\magwt(x_i)$ is called the \emph{magnitude weight}. 
The magnitude of $(X,d)$ is defined as $ \mathrm{Mag}(X,d) \coloneqq \sum_{x_i \in X} \magwt(x_i)$.
\end{definition}

Restricted to a finite $X\subset \mathbb{R}^D$ and the inherited metric, the similarity function $\zeta$ can be written as a positive definite (and invertible) matrix $\zeta_X$. In this case, a weighting $\magwt_X = \zeta_X^{-1}\one$ always exists, where where $\one$ is the $\abs{X} \times 1$ column vector of all ones. Magnitude is the sum of all the entries of the weighting vector, i.e., $Mag(X) \coloneqq \one^\ast \magwt_X = \one^\ast \z_X^{-1} \one$. ($d$ is omitted where understood from context.)

Magnitude can be defined with respect to a scale $t \in \mathbb{R}_{+}$, and a \emph{scaled metric space} $(tX, d_t)$, often denoted by $tX$, where $d_t(x,y) \coloneqq t \cdot d(x,y)$.
The corresponding magnitude weighting vector and similarity function are denoted by $\magwt^t_X$, and $\zeta^t$, respectively.

\section{Magnitude Potential Ratio: Definition and Theory}\label{sec:mp}

Since we are interested in finite subspaces of Euclidean spaces (and assume all points are distinct), we can assume that the similarity matrix is positive definite. See Appendix~\ref{app:assumptions} for details.

\begin{definition}[Magnitude Potential (MP)~\cite{meckes2015magnitude}]
For a finite set $X\subset \mathbb{R}^D$ and any point $z \in  \mathbb{R}^D$, the magnitude potential of $z$ with respect to $X$ at scale parameter $t \in \mathbb{R}_+$ is defined as
\begin{equation} \label{eq:def_mp}
    \rho^t(z, X) = \sum_{x \in X} \zeta^t(x,z) \magwt^t_X(x)
\end{equation}
where $\magwt^t_X(x) \in  \mathbb{R}^D$ are the magnitude weighting of data point $x$ within the dataset $X$. As all terms implicitly depend on the same $t$, we often suppress $t$ from the notation.
\end{definition}
Since $\zeta(x,y) = e^{-t \cdot d(x,y)}$ is a positive definite kernel, it defines a reproducing kernel Hilbert space (RKHS)~\cite{aronszajn1950theory}.
The RKHS $\mathcal{H}$ on $X$, is the completion of the linear span of the functions $k_x(\cdot) = \zeta(x, \cdot)$ with respect to the inner product, i.e. $\langle k_x, k_y \rangle_{\mathcal{H}} = \zeta(x, y)$.
Since $X$ is finite and hence compact, by \citet[Theorem~4.1]{meckes2015magnitude} $\rho^t(z, X)$ is the unique 
minimal-norm function in $\mathcal{H}_\zeta$ interpolating $f \equiv 1$ on $X$. That is, the MP is the smoothest function that assigns $1$ to every point in $X$, with its squared norm equal to the magnitude of $X$~\cite[Theorem 3.5]{leinster2016magnitude}.
When the magnitude potential $\rho(z,X) \approx 1$, the point $z$ is well represented by the points of $X$ at scale $t$, and when $\rho(z,X) \approx 0$, it is far from $X$. 

The following theorems state some continuity and limiting properties of the magnitude potential. 

\begin{restatable}{theorem}{Continuity}[Lipschitz Continuity of Magnitude Potential]\label{theo:continuity}
    For a fixed $X \subset \mathbb{R}^D$ and and scale parameter $t \in \mathbb{R}_{+}$, magnitude potential $\rho^t(z, X)$ is Lipschitz continuous in $z \in \mathbb{R}^D$:
    \begin{equation}
    | \rho ^t (z_1, X) - \rho ^t (z_2, X) | \leq t \cdot d(z_1,z_2)\cdot \|\magwt_X\|_1 \qquad z_1, z_2 \in \mathbb{R}^D. 
    \end{equation}
\end{restatable} 

With $t$ at limits of $0$ and $\infty$, magnitude potential has a set of natural properties: 

\begin{restatable}{theorem}{ScaleLimit}[Scale Limits of Magnitude Potential]\label{theo:scale_limit}
    Let $X \subset \mathbb{R}^D$ 
    and $t \in \mathbb{R}_{+}$.
    At small scale, as $t \to 0$, we have $\rho^t(z, X) \to 1$ for all $z \in \mathbb{R}^D$, and consequently $\operatorname{Var}_z[\rho^t(z,X)] \to 0$.
    At large scale, as $t \to \infty$, we have $\frac{\rho^t(z, X)}{\zeta(x^*, z)} \to 1$ for $z \notin X$ where $x^* = \arg\min_{x \in X} d(x,z)$, as all non-nearest-neighbour kernel values decay to zero exponentially faster.
    \end{restatable}

Theorem~\ref{theo:scale_limit} indicates that in a large $t$ scheme, where the data is scattered, the dominant behavior is entirely based on the distance to the nearest neighbour $x^*$.
Moreover, as $t$ goes to zero, the variance of the dataset itself collapses. The following corollary formalizes the behavior of the magnitude potential in the case where the variance approaches $0$ for the points within $X$:

\begin{restatable}{corollary}{PotentialUnderCollapse}[Magnitude Potential under Collapse]
\label{theo:potential_collapse}
    Let $X \subset \mathbb{R}^D$ be a finite set with positive definite similarity matrix (i.e., $\zeta_X \succ 0$) with mean $ \mu = \frac{1}{|X|}\sum_{x \in X} x$, and within-set variance $V_X := \frac{1}{|X|}\sum_{x \in X} d(x,\mu)^2$. 
    For any $z \in X$, define the leave-one-out set as $X^{(z)} = X \setminus \{z\}$.
    Fix $t > 0$. As $V_X \to 0$ while $\zeta_X \succ 0$,
    \begin{equation}\label{eq:within_class_variance}
        \operatorname{Var}_{z \in X}\bigl[\rho^t(z, X^{(z)})\bigr] \to 0,  \quad  \rho^t(z, X^{(z)}) \to 1 \quad \forall z \in X.
    \end{equation}
    and, $\rho^t(z, X) \to \exp(-t \cdot d(z,\mu))$ for all  $z \notin X$.
\end{restatable}

\paragraph{Magnitude potentials with respect to training data $\mathcal{D}$.} We are interested in the MP with respect to the training set $\mathcal{D}$. For general points $z\in \mathbb{R}^D$, $\rho^t(z, \mathcal{D})$ suffices. When we need to understand how well a training point $x\in \mathcal{D}$ is represented within $\mathcal{D}$, $\rho^t(x, \mathcal{D})$ is always $1$ and uninformative. We thus use $\rho^t(x, \mathcal{D}\setminus \{x\})$ when examining potential of training points. The following theorem gives an analytic expression for it in terms of the regular magnitude weights so that it does not require recomputation of full magnitude weights each time.

\begin{restatable}{theorem}{mpWithinClass}[Leave One Out Magnitude Potential]\label{theo:mp_loo}
For any $x_r \in X$ and scale parameter $t \in \mathbb{R}_{+}$, by definition we have $\rho^t(x_r, X) = 1$. Furthermore, the magnitude potential of $x_r$ with respect to the set $X^{(x_r)}=X\setminus\{x_r\}$ is given by:
\begin{equation}
    \rho^t(x_r, X^{(x_r)}) = 1 - \frac{\magwt_X(x_r)}{[\zeta_X^{-1}]_{rr}}, \label{eq:mp_loo}
\end{equation}
where $\magwt_r$ is the magnitude weight of $x_r$ in $X$ and $[\zeta_X^{-1}]_{rr}$ is the $r$-th diagonal entry of $\zeta_X^{-1}$.
\end{restatable}
Theorem~\ref{theo:mp_loo} is essential to the efficiency of the method, as it computes the magnitude potential using only the inverse of matrix $\zeta_X$, rather than requiring the inversion of $\zeta_{X^{(z)}}$ for each $z \in X$.

\subsection{Magnitude Potential Ratio (MP ratio)}

We introduce the magnitude potential ratio to quantify the representation of the point within class $c$ versus the entire dataset:

\begin{definition}[Magnitude Potential Ratio $\mathcal{R}$]\label{def:mp_ratio}
For a finite set $X\subset \mathbb{R}^D$, let $X_c \subset X$ be the subset of $X$, i.e., points belonging to class $c$. The magnitude potential of a given data point $z \in \mathbb{R}^D$ with scale parameter $t \in \mathbb{R}_+$ is defined as
\begin{equation} \label{eq:def_mp_ratio}
   \mathcal{R}^t_{X, X_c}(z) = \frac{\rho^t(z, X^{(z)}_c )}{\rho^t(z, X^{(z)})}.
\end{equation}
Therefore, when $z \notin X_c$, and $z \notin X$ consequently, we have $ \mathcal{R}^t_{X, X_c}(z) = \frac{\rho^t(z, X_c )}{\rho^t(z, X)}$.
\end{definition}

Intuitively, the MP ratio quantifies whether the point lies within a class region or mixes with the rest of the classes. %
In case of training data point $z \in \mathcal{X}$, the MP ratio is computed using the leave-one-out form as described in the Definition~\ref{def:mp_ratio}. As described previously, this is due to the fact that the MP for points belonging to that set equals one, i.e., $\rho^t(z, \mathcal{X}) = 1$ by Theorem~\ref{theo:mp_loo}.
However, for each validation point $z \in \mathcal{V}$, since we have $\mathcal{X}^{(z)} = \mathcal{X}$ and $\mathcal{X}_c^{(z)} = \mathcal{X}_c$, the MP ratio is computed as $\mathcal{R}^t_{\mathcal{X}, \mathcal{X}_c}(z) = \frac{\rho^t(z, \mathcal{X}_c )}{\rho^t(z, \mathcal{X})}$.

We have already seen in Figure~\ref{fig:mp_vs_kde} that while somewhat similar to KDE, magnitude potential automatically corrects for redundancy by down-weighting dense clusters and up-weighting geometrically isolated points. This property is the geometric intuition of our work.  More comparison with KDE and $k$-nearest-neighbour can be seen in Appendix~\ref{sec:mp-kde-figs}.

In the experimental analysis in the following sections, we compute the MP ratio in the model's latent space, particularly the pre-softmax logit activations, denoted $\phi(\mathcal{X})$ and $\phi(\mathcal{V})$  for the training and validation sets, respectively. 
All distances $d(\cdot,\cdot)$ in Definition~\ref{def:mp_ratio} are therefore computed in this representation space. In the following, we demonstrate that $\mathcal{R}$ captures representational structure at a per-sample level, can be applied independently of architecture, and is sensitive to changes that are not detected easily by aggregate predictive quantities.

\section{MP ratio and memorization scores} \label{sec:per-sample_mem}
\citet{feldman2020does} defines the memorization value of a training data point $x_i$ as the marginal effect that removing the point from the training dataset has on the probability of classifying it correctly. Formally, it is defined as $\mathrm{mem}(A, X, i) = \Pr_{h \sim A(X)}[h(x_i) = y_i] - \Pr_{h \sim A(X^{(i)})}[h(x_i) = y_i]$.
Estimating this score empirically requires an expensive training process: one must train thousands of models~\cite{feldman2020neural}, including and excluding each point. MP ratio provides a per-sample score at any single checkpoint, requires no retraining, and is applicable to both training and validation (unseen samples) datasets.
Feldman's score, similar to MP Ratio, is typically non-negative, but MP Ratio is not bounded above by one. Despite the differences, both methods measure related quantities. We validate this connection empirically on the CIFAR-100 dataset.

\textbf{Dataset and memorization scores.}
We use CIFAR-100 \citep{krizhevsky2009learning}, a 100-class benchmark of
50,000 training images ($32{\times}32$).
Memorization scores are taken directly from \citet{feldman2020neural}, who
provide a per-sample score $\mu_i \in [0,1]$ measuring how much a training point is individually memorized by models trained on the full dataset. The estimates are computed by training 4,000 ResNet-50 models, each on a random $70\%$ subset of the full CIFAR-100 training set.

\textbf{Model.}
We use ResNet-50 \citep{he2016deep}, trained on the full 50,000 training points with the standard cross-entropy loss and no data augmentation or weight decay.
The model reaches $100\%$ training accuracy and $74.3\%$ validation accuracy at epoch $160$, at which point we extract logits for all training samples.

The MP ratio is computed for all 50,000 training points using a random subsample of $200$ data points per class ($40\%$ of the per-class budget of 500) as the set $X$ for the closed-form leave-one-out by Theorem~\ref{theo:mp_loo}. The MP ratio for the remaining points (both training and validation) is computed relative to the entire sub-sample sets without requiring the Leave-One-Out method. 
Table~\ref{tab:feldman_spearman} reports Spearman correlation ($\rho$) between MP ratio and Feldman memorization scores on CIFAR-100. 
The correlation is negative: a high MP ratio indicates strong geometric coverage by the training set, whereas a high memorization score indicates a point that the model has memorized rather than generalized.
Highlighted rows indicate absolute correlations greater than $0.70$, with the highest at $t = 0.2$.
Smaller $t$ values are more sensitive to the most memorized points.

As the model generalizes well for most points, the data shows a large cluster of low memorization and high MP ratio scores. And increasingly more memorised points -- naturally fewer in number -- are found at decreasing MP ratio scores (Figure~\ref{fig:feldman_kde}, Appendix~\ref{app:per_sample_mem}). 
The choice of the scaling parameter $t$ affects the sensitivity of the MP ratio to local geometry, whereas a larger $t$ better captures the least memorized points. 
Averaging over 10 independent random samples reduces noise and consistently improves correlation, as confirmed in Table~\ref{tab:feldman_spearman}. We additionally compare against four single-checkpoint baselines (max softmax probability, logit margin, entropy, and cross-entropy loss), reporting Spearman correlation and the fraction of the 500 most-memorized points by Feldman flagged by each measure.

\begin{table}[h]
    \centering
    \begin{tabular}{lcc}
        \toprule
        Measure & Spearman $\rho$ & Top-500 overlap \\
        \midrule
        Max softmax probability & $-0.507$ & 3.8\% (19/500) \\
        Logit margin             & $-0.655$ & 3.6\% (18/500) \\
        Entropy                  & $0.661$  & 3.8\% (19/500) \\
        Cross-entropy loss       & $0.507$  & 3.8\% (19/500) \\
        \midrule
        MP ratio ($t{=}0.05$, 10-seed avg) & $-0.620$ & 12.0\% (60/500) \\
        MP ratio ($t{=}0.10$, 10-seed avg) & $-0.703$ & 17.0\% (85/500) \\
        MP ratio ($t{=}0.15$, 10-seed avg) & $\mathbf{-0.751}$ & \textbf{17.6\% (88/500)} \\
        MP ratio ($t{=}0.20$, 10-seed avg) & $\mathbf{-0.763}$ & \textbf{17.6\% (88/500)} \\
        MP ratio ($t{=}0.25$, 10-seed avg) & $-0.675$ & 10.0\% (50/500) \\
        \bottomrule
    \end{tabular}
    \vspace{5pt}
    \caption{
    Spearman correlation between MP ratio and Feldman memorization scores on CIFAR-100 across multiple scales.
    Highlighted rows indicate the strongest correlations. Averaging over random seeds consistently improves correlation. 
    }
    \label{tab:feldman_spearman}
\end{table}

\section{Training Dynamics} \label{sec:training_dynamics}
Beyond per-sample scoring, the MP ratio, aggregated over the data points, tracks the evolution of the representational geometry throughout training. We compute MP ratio for all training and validation data points, both relative to the training dataset.
\subsection{Fitting and Decision Boundary} \label{sec:fitting}
We use a controlled 2D setting for direct visualization of the relationship between $\mathcal{R}$ and the decision boundary. 
The goal is to show that $\mathcal{R}$ detects the geometric transition from simple linear boundaries to smooth class-specific curves and also overfitting.

\textbf{Dataset:}
Three isotropic Gaussian classes in $\mathbb{R}^2$ ($K=3$, $d=2$), with 80 samples per class (240 total), split 70/30 into train/validation.

\textbf{Model:}
A 2-hidden-layer MLP with widths $d \to 64 \to 32 \to K$, ReLU activations.
Trained with Adam ($\eta=10^{-3}$, $\lambda=0$) using cross-entropy loss for 3\,000 epochs.

Throughout the training, the validation accuracy remains near $100\%$, however, the MP ratio detects a non-trivial structural change in the decision boundary and the overfitting phase. 

The model initially maintains simple, approximately linear decision boundaries that achieve high accuracy due to the well-separated class geometry, while the standard deviation of validation MP ratio remains steady. Then this standard deviation drops, while the mean MP ratio stays high indicating that the model has begun assigning uniformly high MP ratios to validation samples.
This change aligns with the decision boundaries switching into smooth, class-specific curves that tightly conform to each class region (Figure~\ref{fig:toy_boundaries}).
At this stage (epochs 150 to 250), the standard deviation is low, indicating a stable representation status.
After this stable phase (epoch 250), while the MP ratio remains high, the standard deviation increases. This increase indicates that the model's representation space assigns increasingly non-uniform MP ratios to validation points: a geometric signal of overfitting. Conceptual Figure~\ref{fig:concept} illustrates the same MP ratio mean and standard deviation with scaling $t= 1.0$.

Further experiments in Appendix~\ref{app:toy_dataset} compare the magnitude potential with the accuracy, weighting norms, and relative flatness proposed by~\citet{petzka2021relative}.
The results show that the MP ratio provides early signs of phase transitions that other metrics fail to capture.
Heatmaps of the magnitude potential for the within-class and the entire dataset, as well as the MP ratio, are also provided.

\begin{figure}[t]
  \begin{center}
    \includegraphics[width=\linewidth]{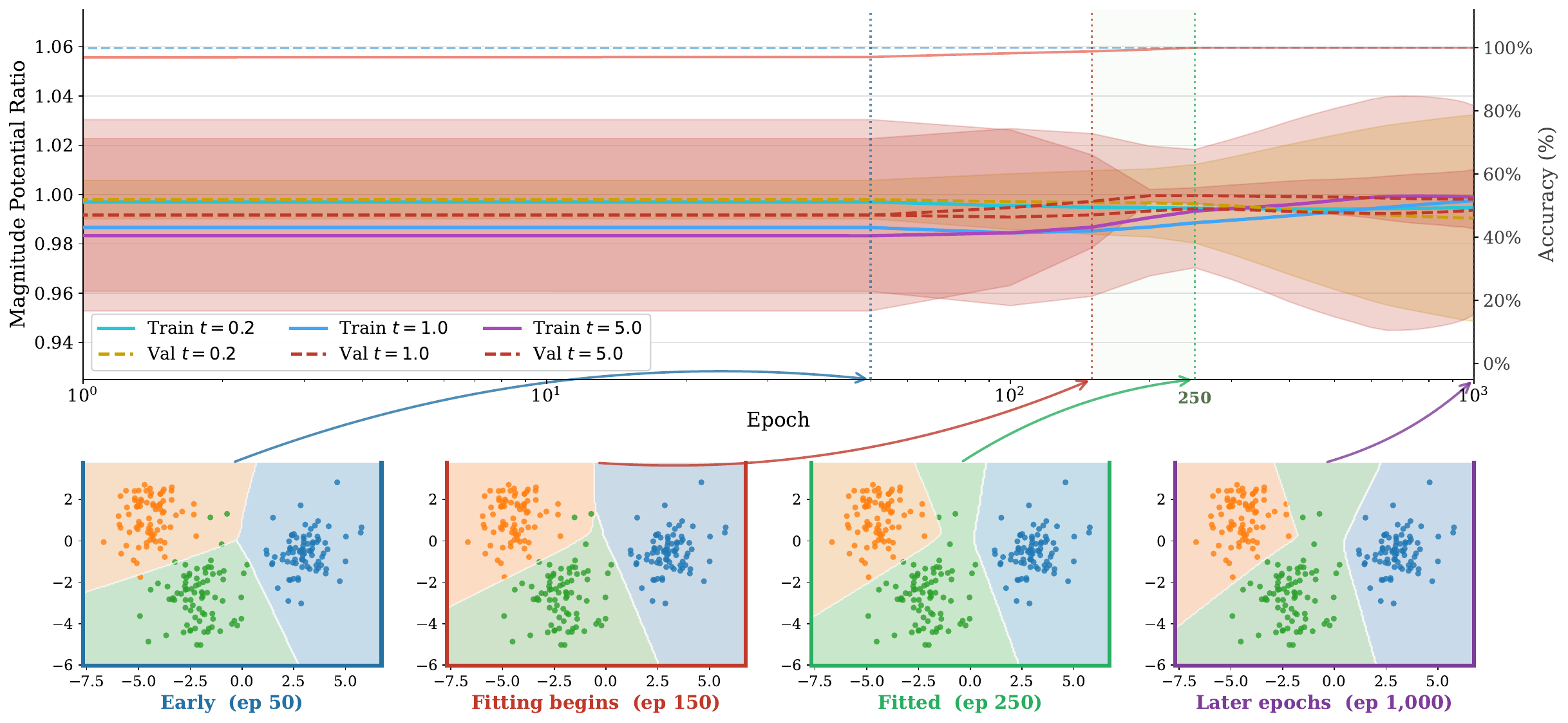}
  \end{center}
  \caption{Validation MP Ratio ($t \in \{0.2, 1.0, 5.0\}$) and decision boundaries for the 2D Toy Gaussian dataset.
  The shaded band around the MP Ratio curve represents the standard deviation of the ratio across validation samples at each epoch, matching the mean line color. 
  Despite near-perfect validation accuracy throughout training, the collapse of the validation MP ratio standard deviation (shaded band) indicates a structural transition in which the model shifts from simple planar boundaries to smooth class-specific curves. }
  \label{fig:toy_boundaries}
\end{figure}

\subsection{Grokking in Modular Addition Datasets}\label{sec:grokking}
Grokking is studied as a more challenging setting where generalization is delayed long after memorization. 
Detecting this phase transition is practically valuable: it indicates whether continued training is driving generalization, stabilizing, or overfitting, and so helps avoid both premature termination and unnecessary overtraining.
We demonstrate that the MP ratio signals the transition to higher accuracy and generalization, as well as the model stabilization phase, in which the weight norm and losses converge.
Unlike existing methods for studying grokking, the MP ratio does not require any assumption on the model architecture or loss function.
In the following, we present the results on a standard grokking benchmark, i.e., the modular addition task, for a 2-layer MLP with an MSE loss.
On the same modular addition task, results replicating the same dynamics for a 1-layer Transformer with cross-entropy loss
~\cite{nanda2023progress}, are provided in Appendix~\ref{app:modular_transformer}.

\textbf{Dataset:}
Modular addition over $\mathbb{Z}_{37}$: the dataset consists of $\frac{p(p+1)}{2} = 703$ ordered pairs $(a,b)$ with $0 \le a \le b < p=37$, labelled by $(a+b)\bmod p$.
Each input is encoded as $\mathbf{x} = \mathbf{e}_a + \mathbf{e}_b \in \mathbb{R}^{37}$ (sum of two one-hot vectors). The data is split 70/30 into train/test.

\textbf{Model:}
A 2-layer MLP  without bias terms: $\mathbb{R}^p \xrightarrow{E} \mathbb{R}^{512} \xrightarrow{\text{LeakyReLU}} \mathbb{R}^p$, with $E \sim \mathcal{N}(0, p^{-1/2})$. Trained with SGD ($\eta=1.0$, $\lambda=10^{-4}$, no momentum) using MSE loss for 210\,000 steps (full-batch).

Figure~\ref{fig:modular_mlp} shows distinct phases in the training, each captured differently by the MP ratio signals.

\textbf{Accuracy transition.}
When validation accuracy begins its sharp rise, both mean and standard deviation of the validation MP ratio ($\sigma(\mathcal{R}_{\mathrm{val}})$ increase as shown in the lower strip of Figure~\ref{subfig:val_mp_ratio}. The increase in the mean validation MP ratio indicates that validation points are more covered by their class training points. The increase in the standard deviation of the validation MP ratio signals the ongoing structural change.

\textbf{Generalization and Early Grokking Detection.}
The rise of the MP ratio for training ($\mathcal{R}_{\mathrm{train}}{\uparrow}$ in Figure~\ref{subfig:train_mp_ratio}) marks the epoch at which training samples begin to consolidate within their class distributions. 
At the same time, the standard deviation of the training MP ratio ($\sigma(\mathcal{R}_{\mathrm{train}})$) decreases as shown in the lower strip of Figure~\ref{subfig:train_mp_ratio}.
This indicates that during this phase, the model keeps the training data points within a class, distinguished from the rest, in the logit space. 
The weight norm provides a slightly earlier signal through its continued decrease. Training and validation losses continue to decrease gradually, with no discontinuities or other signs of change. 
Together, these geometric changes in logit-space representations detected by the MP ratio are consistent with signals from other leading indicators of models' generalization, preceding the transition in validation accuracy.

\textbf{Stabilization.}
As validation accuracy approaches $100\%$, all signals converge: the MP ratios standard deviation for both validation and training decrease (Figure~\ref{subfig:val_mp_ratio} and~\ref{subfig:train_mp_ratio}), both mean MP ratio values, loss functions, and the weight norm all stabilize simultaneously.
This convergence confirms that the model has reached a stable generalizing solution, 
a state that accuracy alone, having saturated almost instantaneously, cannot distinguish from earlier phases.

\begin{figure}[t]
     \centering
     \begin{subfigure}[t]{0.49\textwidth}
         \centering
         \includegraphics[trim={0 28pt 0 4pt}, clip, width=\linewidth]{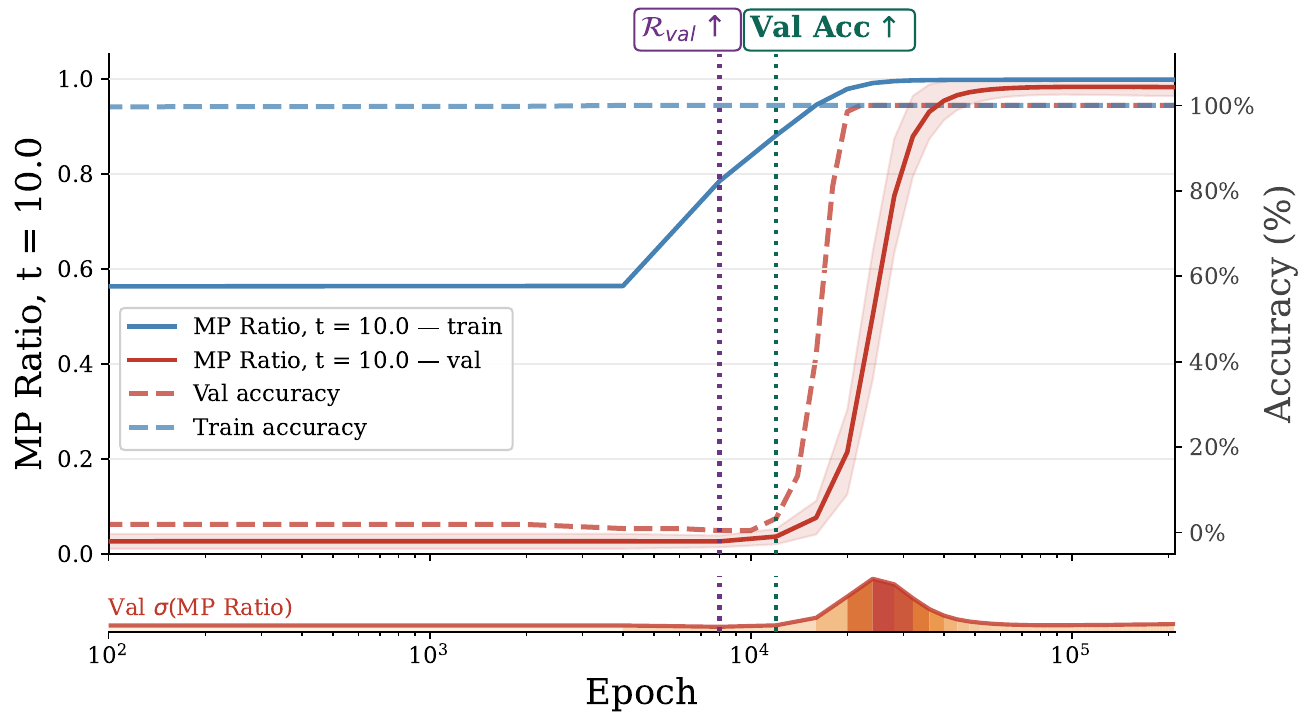}
         \caption{MP ratio vs accuracy functions.}
         \label{subfig:val_mp_ratio}
     \end{subfigure}
     \hfill
     \begin{subfigure}[t]{0.49\textwidth}
         \centering
         \includegraphics[trim={0 8pt 0 4pt}, clip, width=\linewidth]{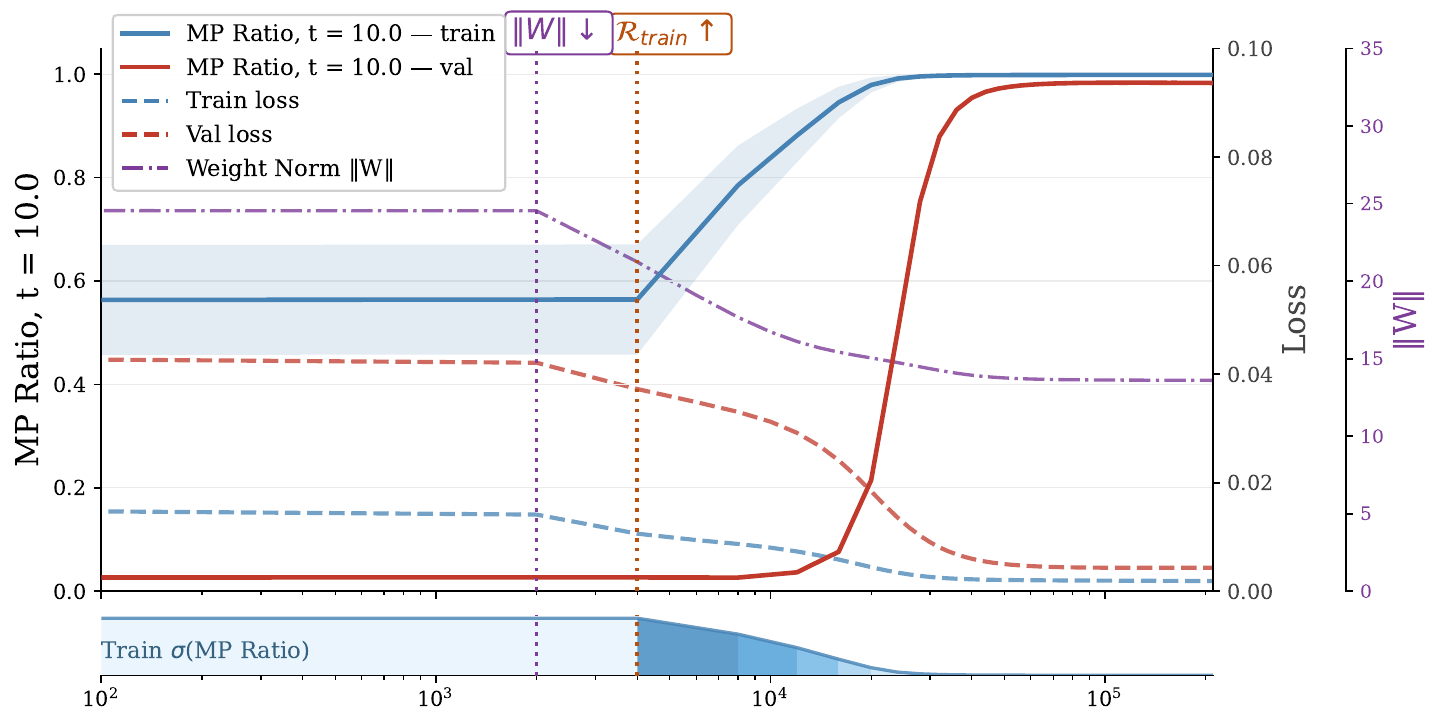} 
         \caption{MP ratio vs loss functions and weight norm.}
         \label{subfig:train_mp_ratio}
     \end{subfigure}
  \caption{Training dynamics over epochs for modular addition $\pmod{37}$ using MLP with MSE Loss.
  \textbf{(a)}~MP ratio, $t=10$) for the training (blue) and validation (red) sets, with accuracy curves overlaid (dashed). The shaded band denotes the standard deviation of the validation MP ratio, and the narrow strip below plots this standard deviation directly.
  Vertical markers: onset of validation MP ratio rise $\mathcal{R}_{\mathrm{val}}{\uparrow}$ (violet) and  validation accuracy rise (teal, $\uparrow$).
  \textbf{(b)}~Train/validation loss(dashed), and weight norm $\|W\|$ (purple).
  Vertical markers: onset of train MP ratio rise $\mathcal{R}_{\mathrm{train}}{\uparrow}$ (orange) and weight norm decreasing $\|W\| {\downarrow}$ (purple). All panels share the same log-scaled epoch axis.} 
  \label{fig:modular_mlp}
\end{figure}

\section{Robustness to Neural Collapse} \label{sec:robustness_nc}
The neural collapse (NC) is an empirical phenomenon where within-class variability of penultimate-layer activations collapses to zero, and class means converge to the vertices of a Simplex Equiangular Tight Frame (ETF)~\cite{papyan2020prevalence}.
This raises a natural question: Is the performance of the logit-layer MP ratio in training dynamics simply a consequence of NC happening, or does it happen even when NC is suppressed?
To test the dependence of these two concepts empirically, we compare a standard model and a model trained with a regulariser that suppresses NC~\cite{han2025flatness}.

\textbf{Dataset:}
CIFAR-10~\cite{krizhevsky2009learning}, a 10-class image classification benchmark consisting of 50,000 training and 10,000 test images of size $32{\times}32$.
For the \emph{Anti-NC} condition, the training set is split $40,000/10,000$ via stratified sampling, and the \emph{standard} condition uses the official CIFAR-10 settings.

\textbf{Models:}
Both conditions use a ResNet-18 trained with SGD (learning rate $0.01$, momentum $0.9$, batch size $64$).
The \emph{standard} model minimizes cross-entropy alone for 400 epochs (81 snapshots saved every 5 epochs).
The \emph{Anti-NC} model adds the regularizer of \citet{han2025flatness} with coefficient $\lambda=10^{-3}$, which explicitly penalizes neural collapse by pushing class means apart and disrupting within-class feature concentration; it is trained for 250 epochs (51 snapshots).

At each snapshot, MP ratio is computed on a random subsample of 500 samples per class (5,000 total) using the closed-form Leave-One-Out (Theorem~\ref{theo:mp_loo}).
Figure~\ref{fig:cifar_antinc} shows nearly identical MP Ratio trajectories throughout training of both models. The slight difference between trajectories is expected due to random subsampling in the MP ratio computation. These results indicate that the MP ratio captures generalization structure independently of whether neural collapse occurs.

\section{Limitations}\label{sec:limitations}

\textbf{Tuning parameter $t$:}
Parameter $t$ scales the absolute magnitude and MP values. In our experiments, values of $t$ were determined mainly through trial and error. 
Our objective was to see if the magnitude potential, with suitable calibration, can reflect training dynamics. 
This is made less challenging by the fact that, while the numerator and denominator individually scale with $t$, the within-class to global ratio remains relatively stable.
We report empirical analysis of this by experimenting over different values of $t$. Across the tested values, Spearman correlation stays within the range $-0.62$ to $-0.76$ and the conclusion (strong negative correlation with memorization) does not flip or vanish at any tested value (Table~\ref{tab:feldman_spearman} and; the neural-collapse experiments similarly show qualitatively similar trajectories across all tested $t$ (Fig~\ref{fig:feldman_kde}- \ref{fig:feldman_scatter_errbar}). 
Nonetheless, it will be good to have a principled analytic method for determining reasonable $t$ in various circumstances, but this kind of result is difficult to derive. Additionally, the correct scale depends on the problem and on the measurement objectives. 

\textbf{Grokking setups:}
Grokking and related experiments are expensive due to the long training run; as such, they are usually carried out on smaller setups than full-scale ML applications. We followed the standard experimental setups in this area. It would be useful to verify in large, overparameterized models, but that is unfortunately impractical.

\textbf{Computational cost:}
Theorem~\ref{theo:mp_loo} avoids inverting a separate Leave-One-Out matrix per point. Instead, it requires only one inversion for each class and one for the entire dataset similarity matrix.
However, the computation of the magnitude still requires matrix inversion, which is computationally expensive for large training sets. This is a one-time cost for any particular point cloud, but it can still be a problem in a large training set. Fortunately, random samples from the training set perform reasonably well in experiments. Sampling-based approximations, as described in~\cite{andreeva2025approximating}, are likely to perform well with smaller samples, but this remains to be studied.

\begin{figure}[t]
    \centering
    \begin{subfigure}[b]{0.32 \textwidth}
        \includegraphics[width=\textwidth]{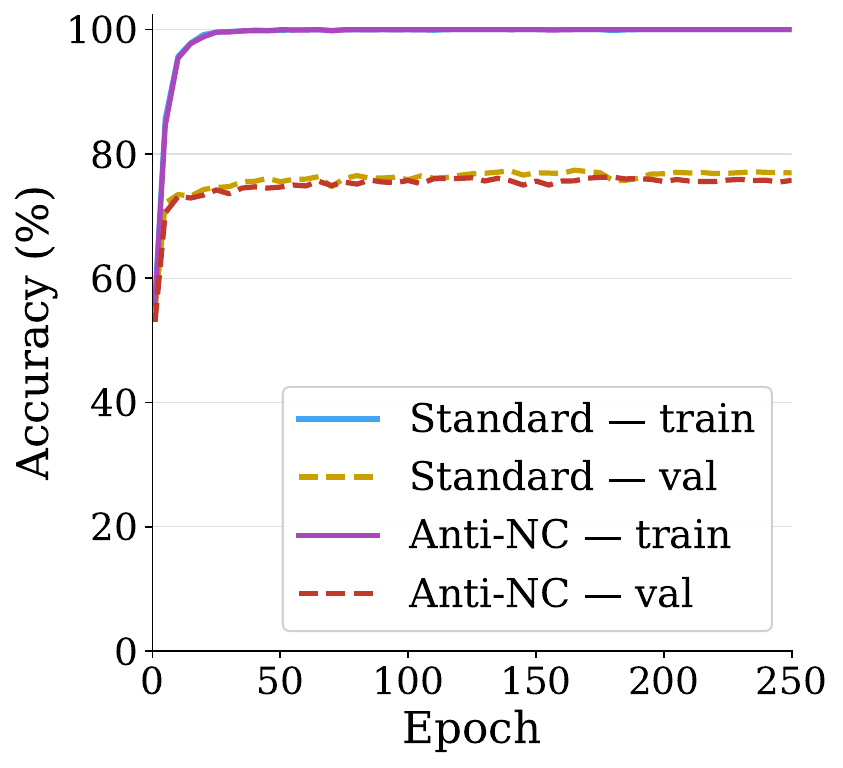}
        \caption{Accuracy}
        \label{fig:cifar_acc}
    \end{subfigure}
    \hfill
    \begin{subfigure}[b]{0.32 \textwidth}
        \includegraphics[width=\textwidth]{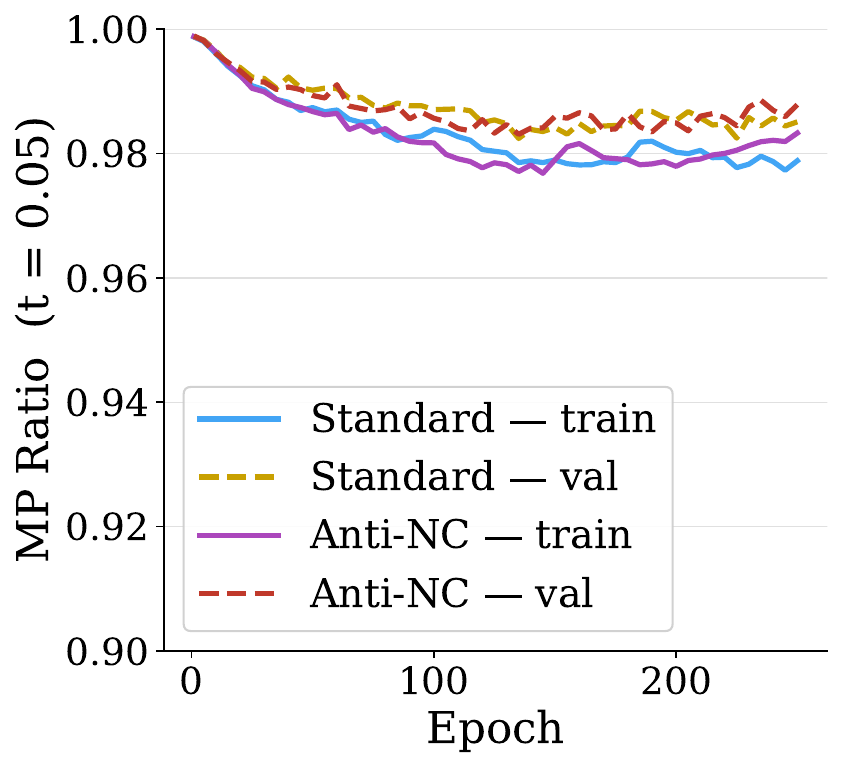}
        \caption{MP Ratio ($t=0.05$)}
        \label{fig:cifar_mp}
    \end{subfigure}
    \hfill
    \begin{subfigure}[b]{0.32 \textwidth}
        \includegraphics[width=\textwidth]{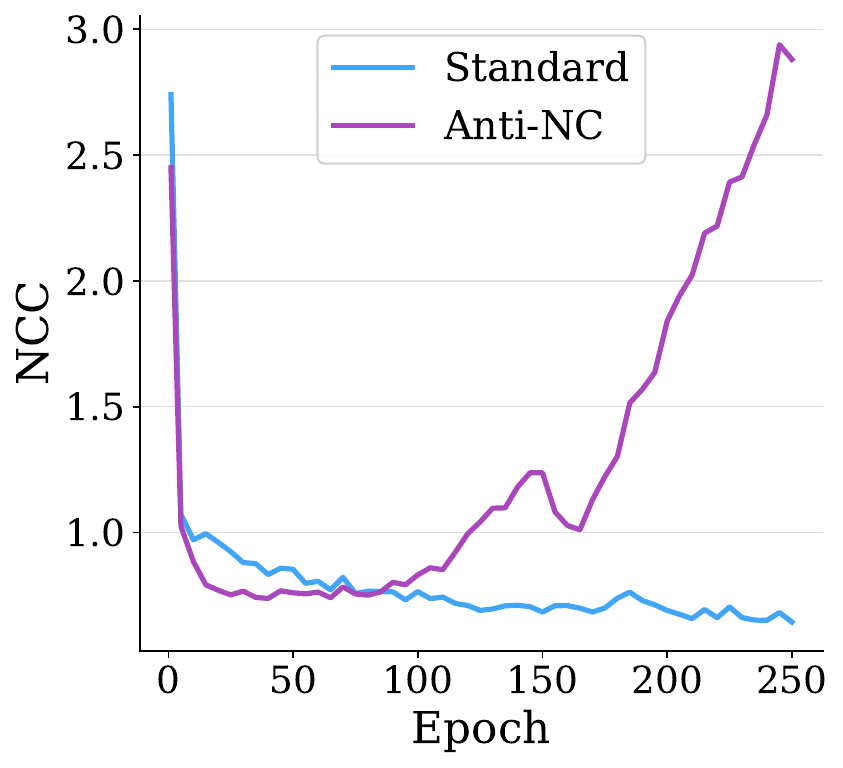}
        \caption{Nearest Class Centre}
        \label{fig:cifar_ncc}
    \end{subfigure}
    \caption{MP Ratio in Absence of Neural Collapse in CIFAR-10.
        Standard training converges to neural collapse while Anti-NC regularization actively suppresses it. Yet both models exhibit nearly identical MP Ratio trajectories throughout training.
        }
    \label{fig:cifar_antinc}
\end{figure}

\section{Related Work}

It was shown by \citet{zhang2017understanding} that large CNNs can memorize entire random training data, and by \citet{belkin2019reconciling} that increasing network size undergoes double descent, where overparameterized networks show good generalization. These experiments prove that the generalization behavior of neural networks is more complex than what is suggested by traditional learning theory. Since then, memorization~\cite{feldman2020neural}, neural collapse~\cite{papyan2020prevalence}, grokking~\cite{power2022grokking}, and similar phenomena have been studied extensively, and observed in many domains~\cite{power2022grokking, liu2023omnigrok, humayun2024deep, zhu2021geometric,zhou2022all}. 

The complexity of the overall training behavior of large networks prevents any one approach or tool from dominating the analysis. Weight decay of networks, such as that from regularization, is found to be closely associated with. grokking~\cite{liu2022towards}, while algorithms-style techniques have been used to argue theoretically about the necessity of memorization~\cite{feldman2020does}. A recent approach based on statistical physics equates memorization and then grokking to a liquid being cooled toward a glass-like state~\cite{zhang2026is}. Various geometric techniques have been applied to study embedding representations, such as in neural collapse~\cite{papyan2020prevalence} and in similar examinations of embedding across layers~\cite{doimo2020hierarchical}. A recent work compares the necessity of neural collapse to that of flatness of loss landscapes in generalization~\cite{han2025flatness}. 

Metric magnitude has also been used to propose a new method for evaluating the generalization of neural networks based on magnitude and the effective number of models~\cite{andreeva2023metric}.
Magnitude dimension, theoretically connected with the persistent homology dimension, proves a new upper bound for the generalization error. Therefore, this links the generalization error to a magnitude-based characteristic of the training trajectories as a whole, in contrast to magnitude potential, which is defined for individual data points, training, and validation, in the network's learned representation.

The current paper belongs to this class of geometric techniques in the study of neural training dynamics. We introduced magnitude potential as a general measure of how well represented a point is by a set (e.g., a training set) and showed that it parallels the other approaches in several respects. This technique should be useful in a wider study of neural networks and machine learning behavior.

\section{Conclusion}

In this work, we propose the Magnitude Potential ratio (MP ratio), a geometric measure grounded in the theory of metric magnitude, that quantifies how well a point is covered by its own class relative to the entire dataset. 
Applied to the model's representation space, the MP ratio captures the geometric relations between unseen points and the training dataset, as well as the geometry within the training dataset.
We prove fundamental properties of the magnitude potential, including Lipschitz continuity, variance collapse under within-class collapse, and a closed-form leave-one-out version.
At the per-sample level, MP ratio strongly correlates with memorization score, without requiring extensive retraining.
Aggregating MP ratio across samples in the training process, detects structural changes in decision boundaries invisible to accuracy, and provides a leading geometric indicator of grokking without requiring any assumptions on the model architecture and loss function.
MP ratio captures generalization structure even when neural collapse is explicitly suppressed.

\begin{ack}
This project was supported by the Royal Academy of Engineering under the Research Fellowship programme.
\end{ack}

\bibliographystyle{plainnat}
\bibliography{ref}

@article{doimo2020hierarchical,
  title={Hierarchical nucleation in deep neural networks},
  author={Doimo, Diego and Glielmo, Aldo and Ansuini, Alessio and Laio, Alessandro},
  journal={Advances in Neural Information Processing Systems},
  volume={33},
  pages={7526--7536},
  year={2020}
}

@article{leinster2013magnitude,
  title={The magnitude of metric spaces},
  author={Leinster, Tom},
  journal={Documenta Mathematica},
  volume={18},
  pages={857--905},
  year={2013}
}

@book{leinster2021entropy,
  title={Entropy and diversity: the axiomatic approach},
  author={Leinster, Tom},
  year={2021},
  publisher={Cambridge university press}
}

@inproceedings{feldman2020does,
  title={Does learning require memorization? a short tale about a long tail},
  author={Feldman, Vitaly},
  booktitle={Proceedings of the 52nd annual ACM SIGACT symposium on theory of computing},
  pages={954--959},
  year={2020}
}

@inproceedings{andreeva2025approximating,
  title={Approximating metric magnitude of point sets},
  author={Andreeva, Rayna and Ward, James and Skraba, Primoz and Gao, Jie and Sarkar, Rik},
  booktitle={Proceedings of the AAAI Conference on Artificial Intelligence},
  volume={39},
  number={15},
  pages={15374--15381},
  year={2025}
}

@inproceedings{andreeva2023metric,
  title={Metric space magnitude and generalisation in neural networks},
  author={Andreeva, Rayna and Limbeck, Katharina and Rieck, Bastian and Sarkar, Rik},
  booktitle={Topological, Algebraic and Geometric Learning Workshops 2023},
  pages={242--253},
  year={2023},
  organization={PMLR}
}

@inproceedings{
torkamani2026magnitude,
title={Magnitude Distance: A Geometric Measure of Dataset Similarity},
author={Sahel Torkamani and Henry Gouk and Rik Sarkar},
booktitle={Forty-third International Conference on Machine Learning},
year={2026},
url={https://openreview.net/forum?id=9ylPoHEKed}
}

@article{meckes2013positive,
  title={Positive definite metric spaces},
  author={Meckes, Mark W},
  journal={Positivity},
  volume={17},
  number={3},
  pages={733--757},
  year={2013},
  publisher={Springer}
}

@inproceedings{
han2025flatness,
title={Flatness is Necessary, Neural Collapse is Not: Rethinking Generalization via Grokking},
author={Ting Han and Linara Adilova and Henning Petzka and Jens Kleesiek and Michael Kamp},
booktitle={The Thirty-ninth Annual Conference on Neural Information Processing Systems},
year={2025},
url={https://openreview.net/forum?id=lbtOctHDQ3}
}

@article{power2022grokking,
  title={Grokking: Generalization beyond overfitting on small algorithmic datasets},
  author={Power, Alethea and Burda, Yuri and Edwards, Harri and Babuschkin, Igor and Misra, Vedant},
  journal={arXiv preprint arXiv:2201.02177},
  year={2022}
}

@inproceedings{
nanda2023progress,
title={Progress measures for grokking via mechanistic interpretability},
author={Neel Nanda and Lawrence Chan and Tom Lieberum and Jess Smith and Jacob Steinhardt},
booktitle={The Eleventh International Conference on Learning Representations },
year={2023},
url={https://openreview.net/forum?id=9XFSbDPmdW}
}

@article{petzka2021relative,
  title={Relative flatness and generalization},
  author={Petzka, Henning and Kamp, Michael and Adilova, Linara and Sminchisescu, Cristian and Boley, Mario},
  journal={Advances in neural information processing systems},
  volume={34},
  pages={18420--18432},
  year={2021}
}

@article{walter2024uncanny,
  title={The uncanny valley: Exploring adversarial robustness from a flatness perspective},
  author={Walter, Nils Philipp and Adilova, Linara and Vreeken, Jilles and Kamp, Michael},
  journal={arXiv preprint arXiv:2405.16918},
  year={2024}
}

@inproceedings{he2016deep,
  title={Deep residual learning for image recognition},
  author={He, Kaiming and Zhang, Xiangyu and Ren, Shaoqing and Sun, Jian},
  booktitle={Proceedings of the IEEE conference on computer vision and pattern recognition},
  pages={770--778},
  year={2016}
}

@article{feldman2020neural,
  title={What neural networks memorize and why: Discovering the long tail via influence estimation},
  author={Feldman, Vitaly and Zhang, Chiyuan},
  journal={Advances in neural information processing systems},
  volume={33},
  pages={2881--2891},
  year={2020}
}

@article{krizhevsky2009learning,
  title={Learning multiple layers of features from tiny images},
  author={Krizhevsky, Alex and Hinton, Geoffrey and others},
  year={2009},
  publisher={Toronto, ON, Canada}
}

@article{meckes2015magnitude,
  title={Magnitude, diversity, capacities, and dimensions of metric spaces},
  author={Meckes, Mark W},
  journal={Potential Analysis},
  volume={42},
  number={2},
  pages={549--572},
  year={2015},
  publisher={Springer}
}

@article{aronszajn1950theory,
  title={Theory of reproducing kernels},
  author={Aronszajn, Nachman},
  journal={Transactions of the American mathematical society},
  volume={68},
  number={3},
  pages={337--404},
  year={1950}
}

@article{leinster2016magnitude,
  title={The magnitude of a metric space: from category theory to geometric measure theory},
  author={Leinster, Tom and Meckes, Mark W},
  journal={arXiv preprint arXiv:1606.00095},
  year={2016}
}

@article{andreeva2024topological,
  title={Topological generalization bounds for discrete-time stochastic optimization algorithms},
  author={Andreeva, Rayna and Dupuis, Benjamin and Sarkar, Rik and Birdal, Tolga and {\c{S}}im{\c{s}}ekli, Umut},
  journal={Advances in Neural Information Processing Systems},
  volume={37},
  pages={4765--4818},
  year={2024}
}

@article{limbeck2024metric,
  title={Metric space magnitude for evaluating the diversity of latent representations},
  author={Limbeck, Katharina and Andreeva, Rayna and Sarkar, Rik and Rieck, Bastian},
  journal={Advances in Neural Information Processing Systems},
  volume={37},
  pages={123911--123953},
  year={2024}
}

@inproceedings{
liu2023omnigrok,
title={Omnigrok: Grokking Beyond Algorithmic Data},
author={Ziming Liu and Eric J Michaud and Max Tegmark},
booktitle={The Eleventh International Conference on Learning Representations },
year={2023},
url={https://openreview.net/forum?id=zDiHoIWa0q1}
}

@inproceedings{
zhang2017understanding,
title={Understanding deep learning requires rethinking generalization},
author={Chiyuan Zhang and Samy Bengio and Moritz Hardt and Benjamin Recht and Oriol Vinyals},
booktitle={International Conference on Learning Representations},
year={2017},
url={https://openreview.net/forum?id=Sy8gdB9xx}
}

@article{
gu2025on,
title={On Memorization in Diffusion Models},
author={Xiangming Gu and Chao Du and Tianyu Pang and Chongxuan Li and Min Lin and Ye Wang},
journal={Transactions on Machine Learning Research},
issn={2835-8856},
year={2025},
url={https://openreview.net/forum?id=D3DBqvSDbj},
note={}
}

@inproceedings{
bonnaire2025why,
title={Why Diffusion Models Don{\textquoteright}t Memorize:  The Role of Implicit Dynamical Regularization in Training},
author={Tony Bonnaire and Rapha{\"e}l Urfin and Giulio Biroli and Marc Mezard},
booktitle={The Thirty-ninth Annual Conference on Neural Information Processing Systems},
year={2025},
url={https://openreview.net/forum?id=BSZqpqgqM0}
}

@inproceedings{
humayun2024deep,
title={Deep Networks Always Grok and Here is Why},
author={Ahmed Imtiaz Humayun and Randall Balestriero and Richard Baraniuk},
booktitle={High-dimensional Learning Dynamics 2024: The Emergence of Structure and Reasoning},
year={2024},
url={https://openreview.net/forum?id=NpufNsg1FP}
}

@article{belkin2019reconciling,
  title={Reconciling modern machine-learning practice and the classical bias--variance trade-off},
  author={Belkin, Mikhail and Hsu, Daniel and Ma, Siyuan and Mandal, Soumik},
  journal={Proceedings of the National Academy of Sciences},
  volume={116},
  number={32},
  pages={15849--15854},
  year={2019},
  publisher={National Academy of Sciences}
}

@article{yu2024generalizablity,
  title={Generalizablity of memorization neural network},
  author={Yu, Lijia and Gao, Xiao-Shan and Zhang, Lijun and Miao, Yibo},
  journal={Advances in Neural Information Processing Systems},
  volume={37},
  pages={113311--113359},
  year={2024}
}

@article{papyan2020prevalence,
  title={Prevalence of neural collapse during the terminal phase of deep learning training},
  author={Papyan, Vardan and Han, XY and Donoho, David L},
  journal={Proceedings of the National Academy of Sciences},
  volume={117},
  number={40},
  pages={24652--24663},
  year={2020},
  publisher={National Academy of Sciences}
}

@inproceedings{
liu2022towards,
title={Towards Understanding Grokking: An Effective Theory of Representation Learning},
author={Ziming Liu and Ouail Kitouni and Niklas Nolte and Eric J Michaud and Max Tegmark and Mike Williams},
booktitle={Advances in Neural Information Processing Systems},
editor={Alice H. Oh and Alekh Agarwal and Danielle Belgrave and Kyunghyun Cho},
year={2022},
url={https://openreview.net/forum?id=6at6rB3IZm}
}

@article{zhu2021geometric,
  title={A geometric analysis of neural collapse with unconstrained features},
  author={Zhu, Zhihui and Ding, Tianyu and Zhou, Jinxin and Li, Xiao and You, Chong and Sulam, Jeremias and Qu, Qing},
  journal={Advances in Neural Information Processing Systems},
  volume={34},
  pages={29820--29834},
  year={2021}
}

@article{zhou2022all,
  title={Are all losses created equal: A neural collapse perspective},
  author={Zhou, Jinxin and You, Chong and Li, Xiao and Liu, Kangning and Liu, Sheng and Qu, Qing and Zhu, Zhihui},
  journal={Advances in Neural Information Processing Systems},
  volume={35},
  pages={31697--31710},
  year={2022}
}

@article{leinster2008euler,
  title={The Euler characteristic of a category},
  author={Leinster, Tom},
  journal={Documenta Mathematica},
  volume={13},
  pages={21--49},
  year={2008}
}

@article{paszke2019pytorch,
  title={Pytorch: An imperative style, high-performance deep learning library},
  author={Paszke, Adam and Gross, Sam and Massa, Francisco and Lerer, Adam and Bradbury, James and Chanan, Gregory and Killeen, Trevor and Lin, Zeming and Gimelshein, Natalia and Antiga, Luca and others},
  journal={Advances in neural information processing systems},
  volume={32},
  year={2019}
}

@inproceedings{
zhang2026is,
title={Is Grokking a Computational Glass Relaxation?},
author={Xiaotian Zhang and Yue Shang and Entao Yang and Ge Zhang},
booktitle={The Thirty-ninth Annual Conference on Neural Information Processing Systems},
year={2026},
url={https://openreview.net/forum?id=Tk5nQnTGmP}
}

@article{pedregosa2011scikit,
  title={Scikit-learn: Machine learning in Python},
  author={Pedregosa, Fabian and Varoquaux, Ga{\"e}l and Gramfort, Alexandre and Michel, Vincent and Thirion, Bertrand and Grisel, Olivier and Blondel, Mathieu and Prettenhofer, Peter and Weiss, Ron and Dubourg, Vincent and others},
  journal={the Journal of machine Learning research},
  volume={12},
  pages={2825--2830},
  year={2011},
  publisher={JMLR. org}
}

\newpage

\appendix
\section{Assumptions} \label{app:assumptions}

It is observed that the magnitude of a finite metric space may be undefined, or smaller than the magnitude of one of its subspaces, or even negative. 
However, positive definite metric spaces have well-defined and unique magnitude weightings as the similarity matrix is invertible. Every subspace of it is also positive definite~\cite{leinster2013magnitude}.

\begin{assumption}[Positive definite Metric Space]
\label{ass:pd}
A finite metric space $(X, d)$ is \textbf{positive definite} 
if $\zeta_X \succ 0$, i.e.\ the similarity matrix is positive 
definite \cite{meckes2013positive, leinster2021entropy}. 
\end{assumption}
The sufficient conditions under which Assumption~\ref{ass:pd} holds, such as the fact that every finite Euclidean Space is positive definite, are explained in the Appendix~\ref{app:assumptions}.

The following are known sufficient conditions under which 
Assumption~\ref{ass:pd} holds~\cite{leinster2013magnitude, leinster2021entropy}:
\begin{enumerate}[label=(\roman*)]
    \item \textbf{Euclidean space:} Every finite subset 
    of $(\mathbb{R}^D, \|\cdot\|_2)$ is positive definite 
    for all $t > 0$. This covers the standard ML setting where data is in $\mathbb{R}^D$.
    \item \textbf{$\ell^1$ metric:} Every finite subset 
    of $(\mathbb{R}^D, \|\cdot\|_1)$ is positive definite 
    for all $t > 0$. 
    \item \textbf{Ultrametric spaces:} Every finite ultrametric space is positive definite with non-negative weights for all $t > 0$.
\end{enumerate}
  
Throughout all proofs, we make use of the following standing result.
In the standard literature, a metric space is understood to be a \emph{set} of distinct points, i.e., without duplicates. \citet{torkamani2026magnitude} extends the magnitude to a finite \emph{collection} of points that may contain duplicates, proving that magnitude is insensitive to redundancy. The weight of each distinct point is distributed among its duplicates. Their result on positive definiteness (proved for Euclidean Space but only used positive definiteness) holds throughout under Assumption~\ref{ass:pd}.
In particular, this ensures the magnitude weights and potential remain 
well-defined and stable, under the collapse of dataset variance or scale parameter limiting to zero.

\section{Proofs of Theorems}\label{app:proofs}
\mpWithinClass*
\begin{proof}[Proof of Theorem~\ref{theo:mp_loo}]
The equation $\rho^t(x_r, X) = 1$ is a direct consequence of the metric magnitude definition. Now, we will prove~\eqref{eq:mp_loo}.
Let $x_r \in X$ and write $X^{(r)} = X \setminus \{x_r\}$ for simplicity. By definition:
\begin{equation}
    \rho^t(x_r, X^{(r)}) 
    = \sum_{i \neq r} \zeta(x_i, x_r)\, \magwt_{X^{(r)}}(x_i)
    \label{eq:loo_pot_def}
\end{equation}
where $\magwt_{X^{(r)}} = \zeta_{X^{(r)}}^{-1}\mathbf{1}$ are the magnitude weights on the reduced set. 
Using the Schur complement identity for the submatrix inverses, we can write the $\zeta_{X^{(r)}}^{-1}$ in terms of $\zeta_X^{-1}$. We partition $\zeta_X$ by isolating index $r$, placed last without loss of generality:
\begin{equation}
    \zeta_X = 
    \begin{pmatrix} A & b \\ b^\top & 1 \end{pmatrix}
\end{equation}
where $A = \zeta_{X^{(r)}} \in \mathbb{R}^{(n-1)
\times(n-1)}$, $b_i = \zeta(x_i, x_r)$ for $i \neq r$, 
and $[\zeta_X]_{rr} = \exp(d(x_r, x_r))=1$. 
Writing $\zeta_X^{-1}$ in the same block structure, we have
\begin{equation}
    \zeta_X^{-1} = \begin{pmatrix} P & q \\ 
    q^\top & s \end{pmatrix}
\end{equation}
where $P \in \mathbb{R}^{(n-1)\times(n-1)}$ is the top-left block of $\zeta_X^{-1}$.
From $\zeta_X \zeta_X^{-1} = I$, the top-left and top-right block equations give respectively:
\begin{align}
    AP + bq^\top &= I \label{eq:block_inversion1}\\
    Aq + bs &= 0 \label{eq:block_inversion2}
\end{align}
From ~\eqref{eq:block_inversion2} we have $A^{-1}b = -q/s$ which by substituting into~\eqref{eq:block_inversion1} we have:
\begin{equation}
    P = A^{-1}(I - bq^\top) = A^{-1} - 
    A^{-1}bq^\top = A^{-1} + \frac{qq^\top}{s}
\end{equation}
Rearranging for $A^{-1} = \zeta_{X^{(r)}}^{-1}$ the 
$(i,j)$-th entry is
satisfies:
\begin{equation}
    [\zeta_{X^{(r)}}^{-1}]_{ij} 
    = [\zeta_X^{-1}]_{ij} 
    - \frac{[\zeta_X^{-1}]_{ir}[\zeta_X^{-1}]_{rj}}
    {[\zeta_X^{-1}]_{rr}}
    \quad \forall\, i,j \neq r.
    \label{eq:schur}
\end{equation}
The weights within $X^{(r)}$ are $\magwt_{X^{(r)}}(x_i) = \sum_{j \neq r}[\zeta_{X^{(r)}}^{-1}]_{ij}$. 
Summing~\eqref{eq:schur} over $j \neq r$ we have:
\begin{align}
    \magwt_{X^{(r)}}(x_i) 
    &= \sum_{j \neq r}[\zeta_X^{-1}]_{ij} 
    - \frac{[\zeta_X^{-1}]_{ir}}{[\zeta_X^{-1}]_{rr}}
    \sum_{j \neq r}[\zeta_X^{-1}]_{rj}.
    \label{eq:loo_weights}
\end{align}
We simplify each sum using $\magwt_X = \zeta_X^{-1}\mathbf{1}$:
\begin{align}
    \sum_{j \neq r}[\zeta_X^{-1}]_{ij} 
    &= \magwt_X(x_i) - [\zeta_X^{-1}]_{ir}, 
    \label{eq:sum1}\\
    \sum_{j \neq r}[\zeta_X^{-1}]_{rj} 
    &= \magwt_X(x_r) - [\zeta_X^{-1}]_{rr}.
    \label{eq:sum2}
\end{align}
Then, substituting~\eqref{eq:sum1} and~\eqref{eq:sum2} 
into~\eqref{eq:loo_weights}:
\begin{align}
    \magwt_{X^{(r)}}(x_i) 
    &= \magwt_X(x_i) - [\zeta_X^{-1}]_{ir} 
    - \frac{[\zeta_X^{-1}]_{ir}( \magwt_X(x_r) - [\zeta_X^{-1}]_{rr})}
    {[\zeta_X^{-1}]_{rr}} \nonumber\\
    &= \magwt_X(x_i) - [\zeta_X^{-1}]_{ir} 
    - \frac{[\zeta_X^{-1}]_{ir} \magwt_X(x_r)}{[\zeta_X^{-1}]_{rr}} 
    + [\zeta_X^{-1}]_{ir} \nonumber\\
    &= \magwt_X(x_i) 
    - \frac{[\zeta_X^{-1}]_{ir}}{[\zeta_X^{-1}]_{rr}} \magwt_X(x_r).
    \label{eq:loo_weight_clean}
\end{align}
Now substitute~\eqref{eq:loo_weight_clean} 
into~\eqref{eq:loo_pot_def}:
\begin{align}
    \rho^t(x_r, X^{(r)}) 
    &= \sum_{i \neq r} \zeta(x_i, x_r) 
    \left( \magwt_X(x_i) 
    - \frac{[\zeta_X^{-1}]_{ir}}{[\zeta_X^{-1}]_{rr}} 
     \magwt_X(x_r) \right) \nonumber\\
    &= \sum_{i \neq r} \zeta(x_i, x_r) \magwt_X(x_i) 
    - \frac{ \magwt_X(x_r) }{[\zeta_X^{-1}]_{rr}} 
    \sum_{i \neq r} \zeta(x_i, x_r)[\zeta_X^{-1}]_{ir}.
    \label{eq:two_sums}
\end{align}
We evaluate each sum separately. First, since 
$\zeta_X \magwt_X = \mathbf{1}$ gives $\sum_i \zeta(x_i, x_r) 
\magwt_X(x_i) = 1$ and $\zeta(x_r, x_r) = 1$:
\begin{equation}
    \sum_{i \neq r} \zeta(x_i, x_r) \magwt_X(x_i) 
    = 1 - \magwt_X(x_r).
    \label{eq:first_sum}
\end{equation}
Also, since $\zeta_X \zeta_X^{-1} = I$ gives 
$\sum_i \zeta(x_r, x_i)[\zeta_X^{-1}]_{ir} 
= [I]_{rr} = 1$ and $\zeta(x_r,x_r)[\zeta_X^{-1}]_{rr} 
= [\zeta_X^{-1}]_{rr}$:
\begin{equation}
    \sum_{i \neq r} \zeta(x_i, x_r)[\zeta_X^{-1}]_{ir} 
    = 1 - [\zeta_X^{-1}]_{rr}.
    \label{eq:second_sum}
\end{equation}
Finally, substituting~\eqref{eq:first_sum} 
and~\eqref{eq:second_sum} into~\eqref{eq:two_sums}, we have:
\begin{align}
    \rho^t(x_r, X^{(r)}) 
    &= (1 -  \magwt_X(x_r)) 
    - \frac{ \magwt_X(x_r) }{[\zeta_X^{-1}]_{rr}}
    (1 - [\zeta_X^{-1}]_{rr}) \nonumber\\
    &= 1 -  \magwt_X(x_r)  
    - \frac{ \magwt_X(x_r) }{[\zeta_X^{-1}]_{rr}} +  \magwt_X(x_r)  
    = 1 - \frac{ \magwt_X(x_r) }{[\zeta_X^{-1}]_{rr}}.
\end{align}
\end{proof}

\Continuity*
\begin{proof}[Proof of Theorem~\ref{theo:continuity}]
By the definition of magnitude potential, for any $z_1, z_2 \in \mathbb{R}^D$ and scale parameter $t \in \mathbb{R}_{+}$, since the weights $\magwt_X(x)$ depend only on $X$, not on $z$, we have:
\begin{align}
    |\rho^t(z_1,X) - \rho^t(z_2,X)| &= | \sum_{x \in X} \big(  \zeta(x,z_1) - \zeta(x,z_2) \big) \magwt_X(x)|
\end{align}
By triangle inequality:
\begin{align}
    |\rho^t(z_1,X) - \rho^t(z_2,X)| &\leq \sum_{x \in X}  | \big(  \zeta(x,z_1) - \zeta(x,z_2) \big) \magwt_X(x)| \nonumber\\
    &= \sum_{x \in X}  | \zeta(x,z_1) - \zeta(x,z_2)| \cdot |\magwt_X(x)| 
\end{align}
Using the Lipschitzness of the exponential function, i.e. $|e^{-ta} - e^{-tb}| \leq t|a-b|$ and the triangle inequality $|d(x,z_1) - d(x,z_2)| \leq d(z_1,z_2)$:
\begin{align}
    |\rho^t(z_1,X) - \rho^t(z_2,X)|  
    &\leq \sum_{x \in X}|\zeta(x,z_1) - \zeta(x,z_2)|\cdot|\magwt_X(x)| \nonumber\\
    &\leq t\cdot d(z_1,z_2)\cdot  \|\magwt_X\|_1.
\end{align}
Finally, note that the theorem holds for all $z_1, z_2 \in \mathbb{R}^D$, whether in 
$X$ or not, In particular, since  $\lim_{z \to x_r}\rho^t(z,X) = 
\sum_x \zeta(x,x_r)\magwt_X(x) = [\zeta_X\magwt_X]_r = 1 = \rho^t(x_r,X)$. 
\end{proof}

\ScaleLimit*
\begin{proof}[Proof of Theorem~\ref{theo:scale_limit}]

\textbf{Part (i): Small scale, $t \to 0$.}

It is similar to proof of \eqref{eq:within_class_variance} of Theorem~\ref{theo:potential_collapse} as $t \to 0$, every kernel value $\zeta(x, z) = e^{-td(x,z)} \to 1$ for all $x \in X$, and $ z \in \mathbb{R}^D$.

\textbf{Part (ii): Large scale, $t \to \infty$, 
$z \notin X$.}
Let $x^* = \arg\min_{x \in X} d(x,z)$ be the nearest neighbour of $z$ in $X$, with $d^* = d(x^*,z) > 0$ as $z \notin X$. 
Then, for any non-nearest neighbour data point $x \neq x^*$, let $\delta_x = d(x,z) - d^* > 0$. Then for each $x \neq x^*$ as $t \to \infty$ we have:
\begin{equation}
    \frac{\zeta(x,z)}{\zeta(x^*,z)} = 
    \frac{e^{-td(x,z)}}{e^{-td^*}} = 
    e^{-t\delta_x} \to 0.
\end{equation}
Therefore, the magnitude potential can be written as:
\begin{align}
    \rho^t(z, X) &= \zeta(x^*,z)\,
    \magwt_X(x^*) + \sum_{x \neq x^*} 
    \zeta(x,z)\,\magwt_X(x) \nonumber\\
    &= \zeta(x^*,z)\left[\magwt_X(x^*) + 
    \sum_{x \neq x^*} e^{-t\delta_x}\,
    \magwt_X(x)\right].
\end{align}
For the weights, as $t \to \infty$ all 
off-diagonal entries of $\zeta_X$ satisfy 
$e^{-td(x_i,x_j)} \to 0$ for distinct 
$x_i \neq x_j$, so $\zeta_X \to I$, 
$\zeta_X^{-1} \to I$, and therefore:
\begin{equation}
    \magwt_X(x) = 
    [\zeta_X^{-1}\mathbf{1}]_x \to 1 
    \quad \forall\, x \in X.
\end{equation}
As $t \to \infty$, all off-diagonal entries of $\zeta_X$, $e^{-td(x_i,x_j)} \to 0$ for $x_i \neq x_j$, so $\zeta_X$ approaches the identity $I$ and $\magwt_X(x) \to 1$ for all $x \in X$.
Therefore:
\begin{equation}
    \frac{\rho^t(z, X)}{\zeta(x^*, z)} \to 1.
\end{equation}
\end{proof}

\PotentialUnderCollapse*
\begin{proof}[Proof of Corollary~\ref{theo:potential_collapse}]
As $V_{X} \to 0$, we have all pairwise distances $d(x_i, x_j)$ converge to zero and every term must vanish, i.e. $d(x, \mu) \to 0$ for all $x \in X$. By the triangle inequality, for any $x_1, x_2 \in X$:
\begin{equation}\label{eq:pairwise_dist_collapse}
    d(x_1, x_2) \leq d(x_1, \mu) + d(x_2, \mu) \to 0.
\end{equation}
For any fixed $z \in \mathbb{R}^D$ and any $x_1, x_2 \in X$, by triangle inequality we have:
\begin{equation}
    d(x_1,z) \leq d(x_1,x_2) + d(x_2,z), 
    \qquad
    d(x_2,z) \leq d(x_1,x_2) + d(x_1,z).
    \label{eq:dist_squeeze}
\end{equation}
As $d(x_1,x_2) \to 0$ by \eqref{eq:pairwise_dist_collapse}, both 
distances and consequently all $d(x,z)$, converge to the same limit 
$d_z := d(\mu,z)$, since $x \to \mu$ for all $x \in X$.
However, as explained, we can ignore duplicates and consider only the set of distinct points, and the weight of each distinct point is distributed among its duplicates. 
\begin{align}
    \rho^t(z, X^{(z)}) 
    &= \sum_{x \in X^{(z)}} 
    \zeta(x,z)\,\magwt_{X^{(z)}}(x) 
    \nonumber\\
    &\to \exp(-td_z)\,\sum_{x \in X^{(z)}} 
    \magwt_{X^{(z)}}(x) 
    \nonumber\\
    &= \exp(-td_z)\, \text{Mag}_{X^{(z)}}(t).
\end{align}

\textbf{Case 1: $z \in X$.}
Since $z \in X$ implies $z \to \mu$ as $V_X \to 0$, we have $d_z = d(\mu,z) \to 0$. Therefore:
\begin{align}
    \rho^t(z, X^{(z)}) 
    &\to \sum_{x \in X^{(z)}} 
    \magwt_{X^{(z)}}(x) 
    = \text{Mag}_{X^{(z)}}(t).
\end{align}
As all points of $X^{(z)}$ collapse to a 
single location, $ \text{Mag}_{X^{(z)}}(t)\to 1$, by definition 
\citep{leinster2013magnitude}. Therefore:
\begin{equation}
    \rho^t(z, X^{(z)}) \to 1 \quad 
    \forall\, z \in X.
\end{equation}
Since all $z \in X$ converge to a same limit:
\begin{equation}
    \mathrm{Var}_{z \in X}
    [\rho^t(z, X^{(z)})] \to 0.
\end{equation}

\textbf{Case 2: $z \notin X$.}
Similar to Case 1, as all points of $X^{(z)}$ collapse to a  single location, $ \text{Mag}_{X}(t)\to 1$. Therefore,
\begin{equation}
    \rho^t(z, ) \to \exp(-td_z), \quad 
    \forall\, z \in X.
\end{equation}
Since all $z \in X$ converge to a same limit:
\begin{equation}
    \mathrm{Var}_{z \in X}
    [\rho^t(z, X)] \to 0.
\end{equation}

\end{proof}

\newpage

\section{Experimental Setup}
\label{app:experimental_setup}

All experiments are implemented in PyTorch~\citep{paszke2019pytorch} with torchvision, NumPy, and scikit-learn~\citep{pedregosa2011scikit}, and run on a single NVIDIA A40 GPU using 4 CPU cores per job. 
System RAM allocations range from 16\,GB for lightweight grokking runs up to 48\,GB for CIFAR-scale training and MP evaluation jobs.

\textbf{Per Sample Memorization Evaluation(Section~\ref{sec:per-sample_mem})}
Each run for a single value of $t$ takes approximately 74 minutes, and a full 10-seed run completes in 12.2 hours in total on a single NVIDIA A40 GPU.

\textbf{Robustness to Neural Collapse (Section~\ref{sec:robustness_nc})}
Processing all snapshots for a single $t$ value takes approximately 7 minutes (Anti-NC, 51 snapshots) and 11 minutes (standard, 81 snapshots) on an NVIDIA A40 GPU.

\section{Magnitude Potential Comparison}\label{sec:mp-kde-figs}

\begin{figure}[h]
    \centering
    \begin{subfigure}[t]{0.26\textwidth}
        \includegraphics[width=\linewidth]{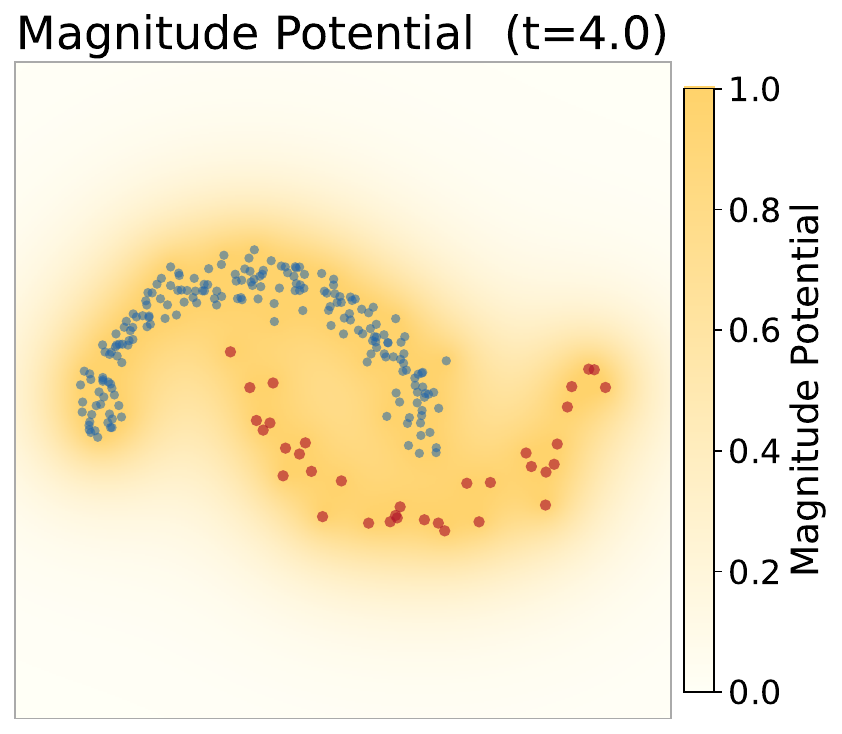}
    \end{subfigure}
    \hfill
    \begin{subfigure}[t]{0.26\textwidth}
        \includegraphics[width=\linewidth]{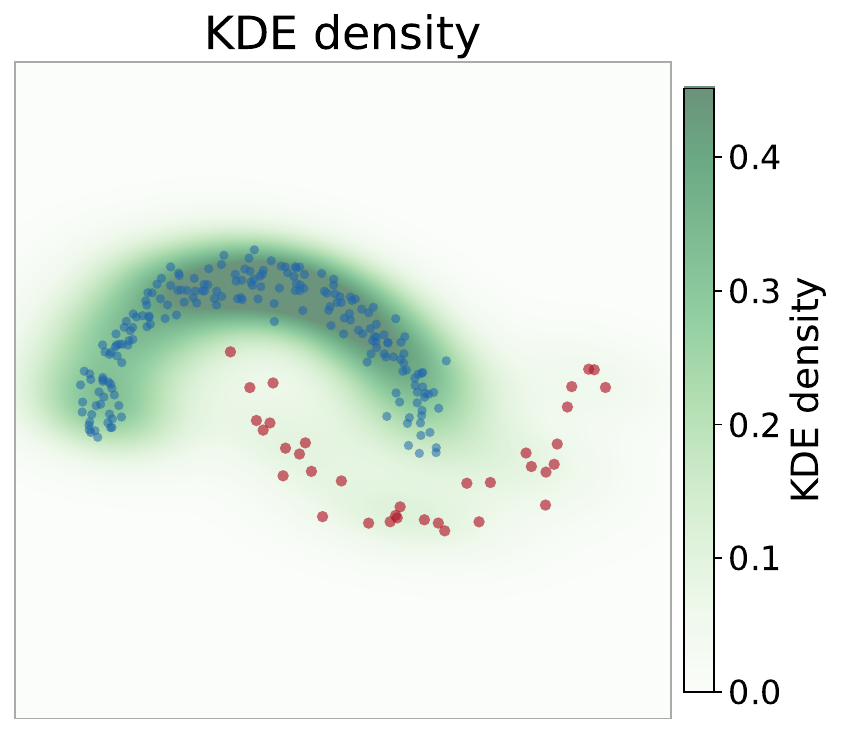}
    \end{subfigure}
    \hfill
    \begin{subfigure}[t]{0.26\textwidth}
        \includegraphics[width=\linewidth]{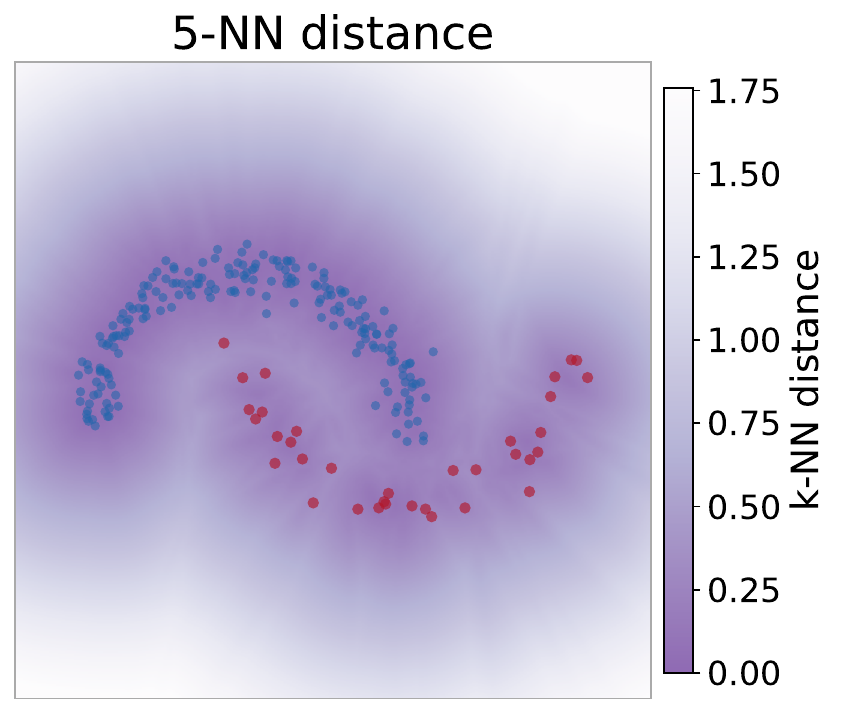}
    \end{subfigure}

    \vspace{0.5em}

    \begin{subfigure}[t]{0.26\textwidth}
        \includegraphics[width=\linewidth]{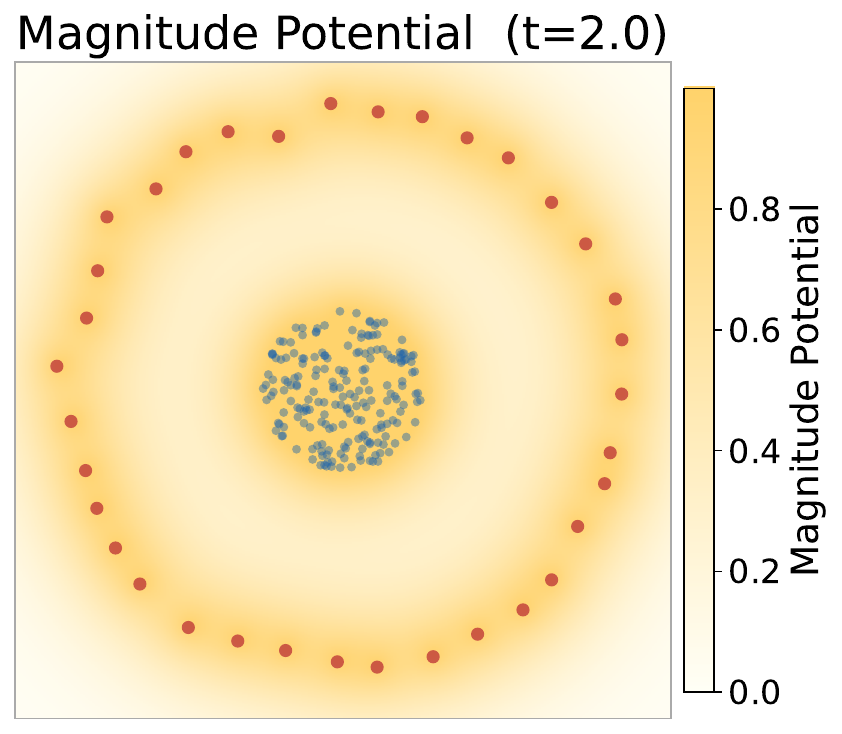}
    \end{subfigure}
    \hfill
    \begin{subfigure}[t]{0.26\textwidth}
        \includegraphics[width=\linewidth]{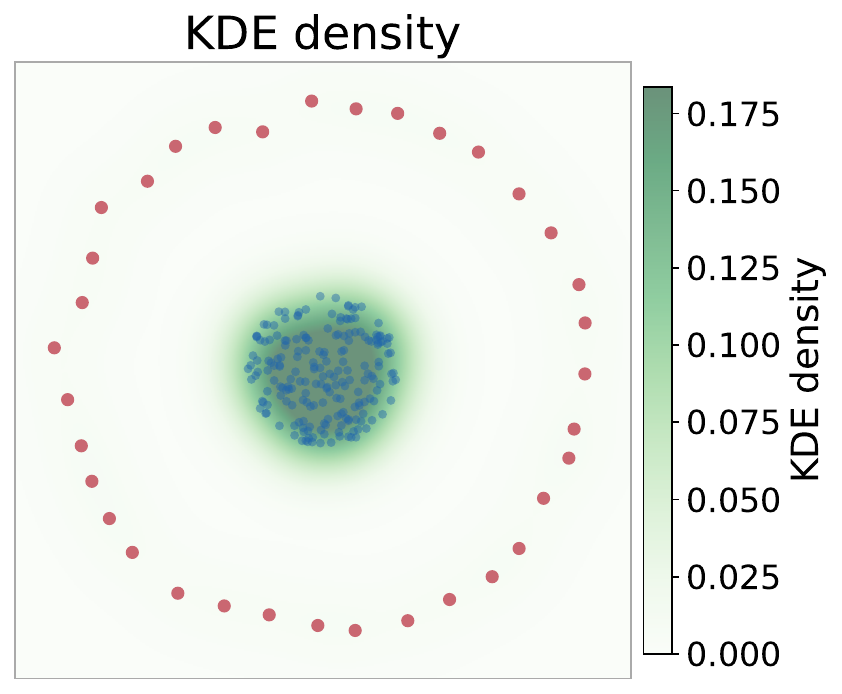}
    \end{subfigure}
    \hfill
    \begin{subfigure}[t]{0.26\textwidth}
        \includegraphics[width=\linewidth]{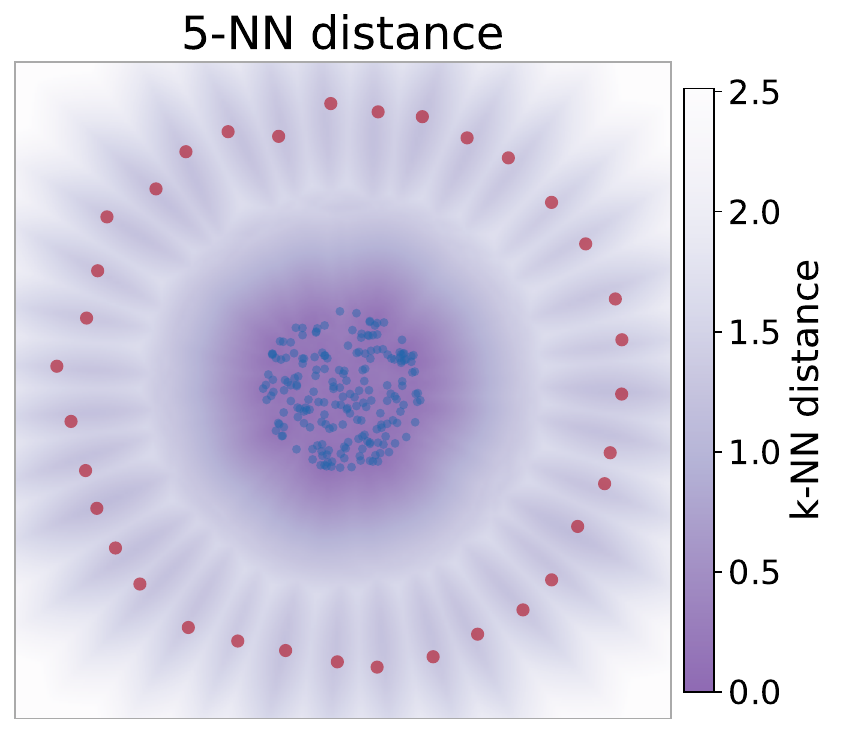}
    \end{subfigure}

    \caption{
        Comparison of Magnitude Potential, KDE, and $k$-NN distance computed on all points in both toy datasets. \emph{Top row}: asymmetric two-moons. \emph{Bottom row}: dense disk and sparse ring. Blue points are dense; red points are sparse. MP assigns high values uniformly across both clusters, regardless of local density, while KDE suppresses the sparse cluster and $k$-NN distance fails to capture the underlying geometry smoothly.
    } 
    \label{fig:toy-comparison}
\end{figure}

\begin{figure}[h]
    \centering
    \begin{subfigure}[t]{0.26\textwidth}
        \includegraphics[width=\linewidth]{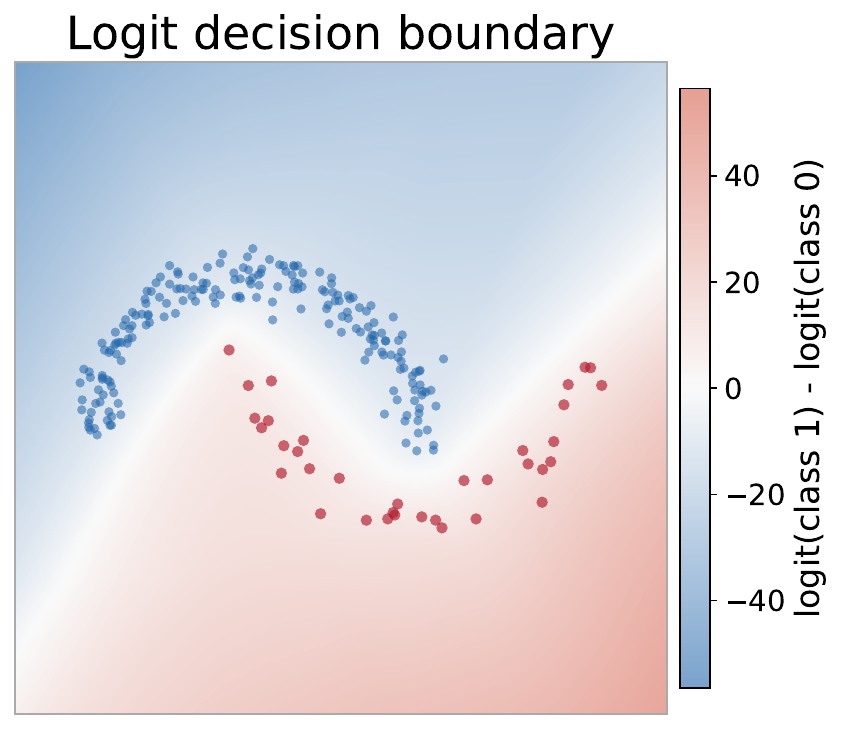}
    \end{subfigure}
    \hfill
    \begin{subfigure}[t]{0.26\textwidth}
        \includegraphics[width=\linewidth]{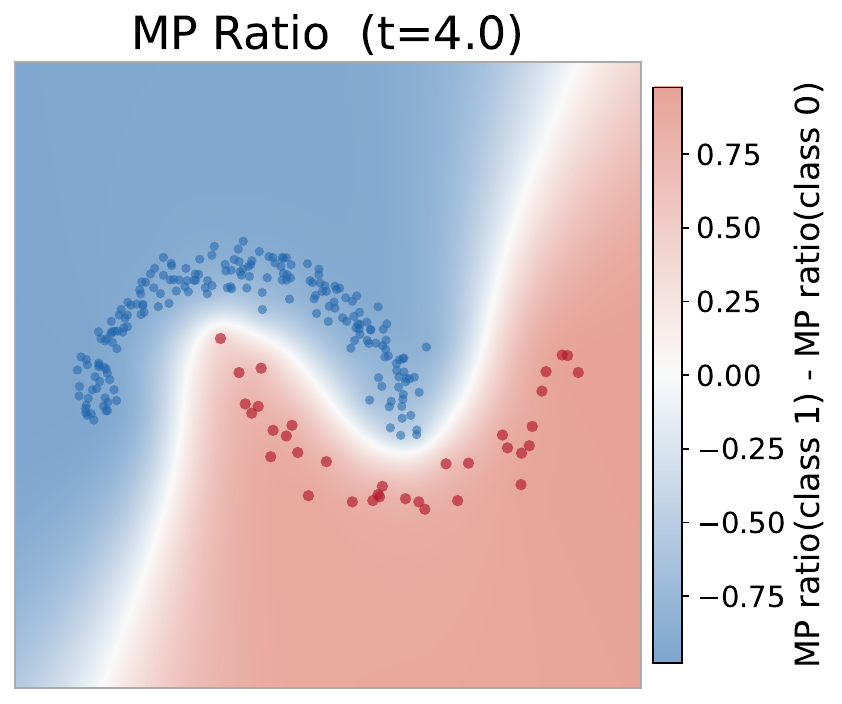}
    \end{subfigure}
    \hfill
    \begin{subfigure}[t]{0.26\textwidth}
        \includegraphics[width=\linewidth]{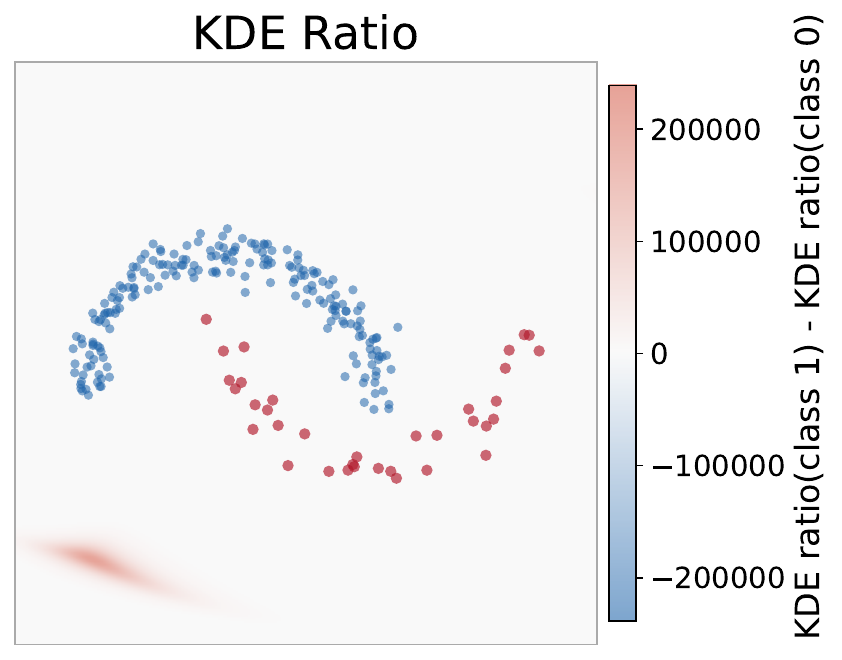}
    \end{subfigure}

    \vspace{0.5em}

    \begin{subfigure}[t]{0.26\textwidth}
        \includegraphics[width=\linewidth]{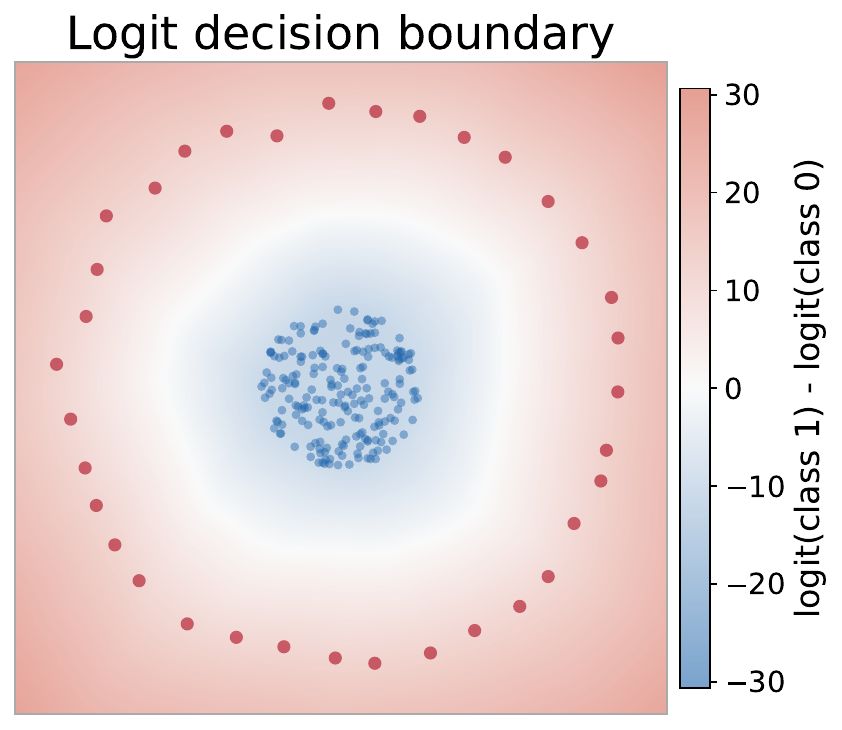}
    \end{subfigure}
    \hfill
    \begin{subfigure}[t]{0.26\textwidth}
        \includegraphics[width=\linewidth]{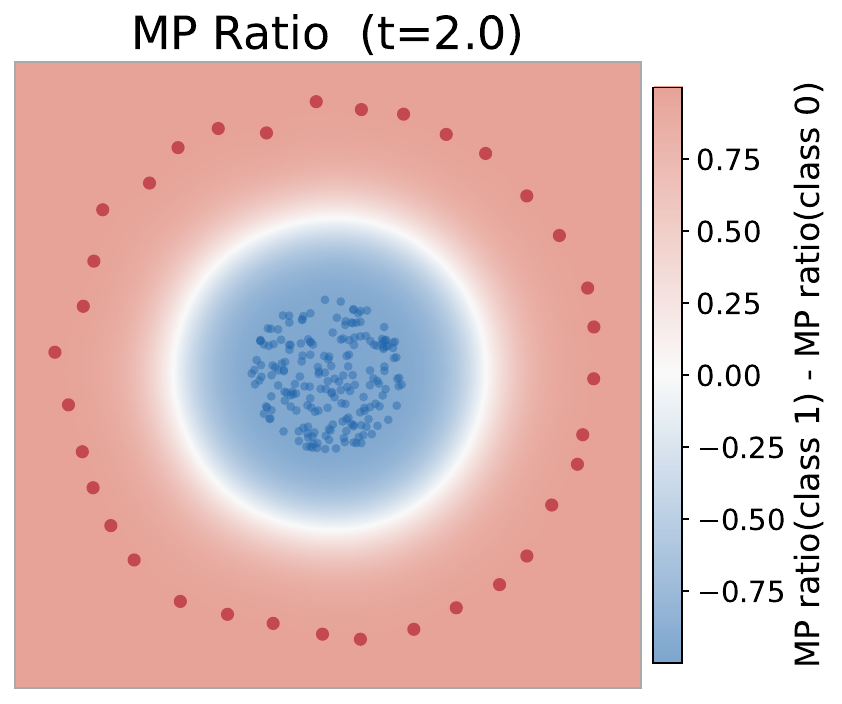}
    \end{subfigure}
    \hfill
    \begin{subfigure}[t]{0.26\textwidth}
        \includegraphics[width=\linewidth]{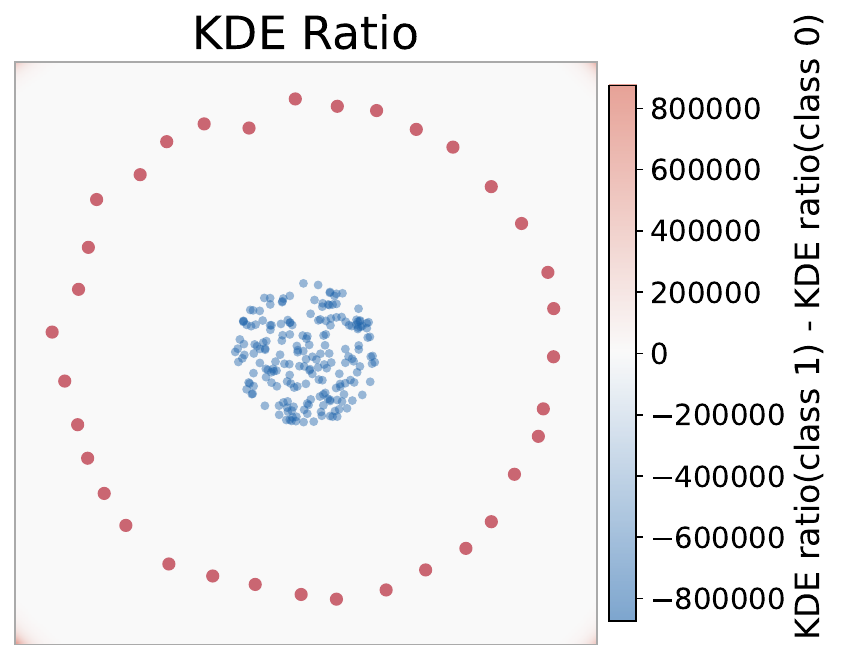}
    \end{subfigure}

    \caption{
        Comparison of logit decision boundary, MP ratio, and KDE ratio on two toy datasets.  \emph{Top row}: asymmetric two-moons. \emph{Bottom row}: dense disk and sparse ring. Blue points are dense; red points are sparse.
        \textbf{Logit} shows the pre-softmax decision boundary learned by a trained classifier.
        \textbf{MP ratio} at every point is computed as defined in Section~\ref{sec:mp}.
        \textbf{KDE ratio} has a similar structure to the MP ratio, but for KDE: the within-class KDE divided by the entire-dataset KDE.
        \textbf{Model:} For each dataset, a two-layer ReLU MLP with 64 hidden units is trained with Adam ($\text{lr}=10^{-3}$, cross-entropy loss) for 2000 epochs until full training accuracy is reached. 
    }
    \label{fig:toy-comparison_ratio}
\end{figure}

\newpage

\section{Fitting and Decision Boundaries -- 2D Toy Dataset}\label{app:toy_dataset}
For comparison, we adopt the \emph{relative flatness} metric proposed by~\citet{petzka2021relative}, which considers the Hessian trace normalized by the weight norm, a quantity theoretically and empirically shown to align with generalization. To make this computable in large state-of-the-art neural networks, we employ the alternative closed-form upper bound introduced by~\citet{walter2024uncanny}, which is valid in the penultimate layer under cross-entropy loss.

While these transitions can be captured by the magnitude potential ratio, other methods can only approve a subset of these observations: generalization is supported by increases in validation accuracy and decreases in validation loss.
The decrease in both losses and relative flatness and an increase in weight norm are all signs of model fitting, but detecting the overfitting early can not be done solely by them. Overfitting can only be observed after a longer time, with these behaviors continuing to occur.

\begin{figure}[htbp]
  \begin{center}
    \includegraphics[width=0.95\linewidth]{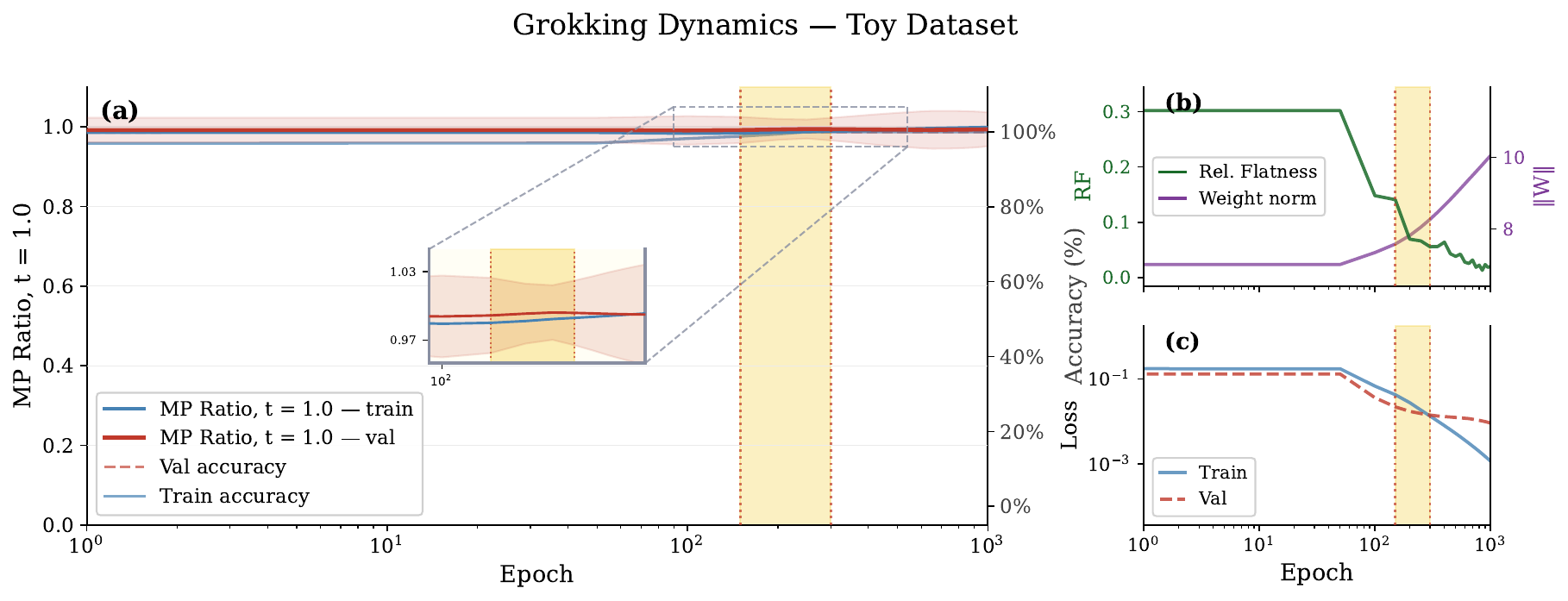}
  \end{center}
  \caption{Validation MP Ratio ($t = 1.0$) for Toy Dataset. The shaded band around the MP Ratio curve represents the standard deviation of the ratio across validation samples at each epoch.
  When the Validation decreases, it is the exact phase that the model changes the decisions on boundary structure completely.
  }
  \label{fig:compact_toy}
\end{figure}

\newpage
\begin{figure}[htbp]
  \begin{center}
    \includegraphics[width=0.90\linewidth]{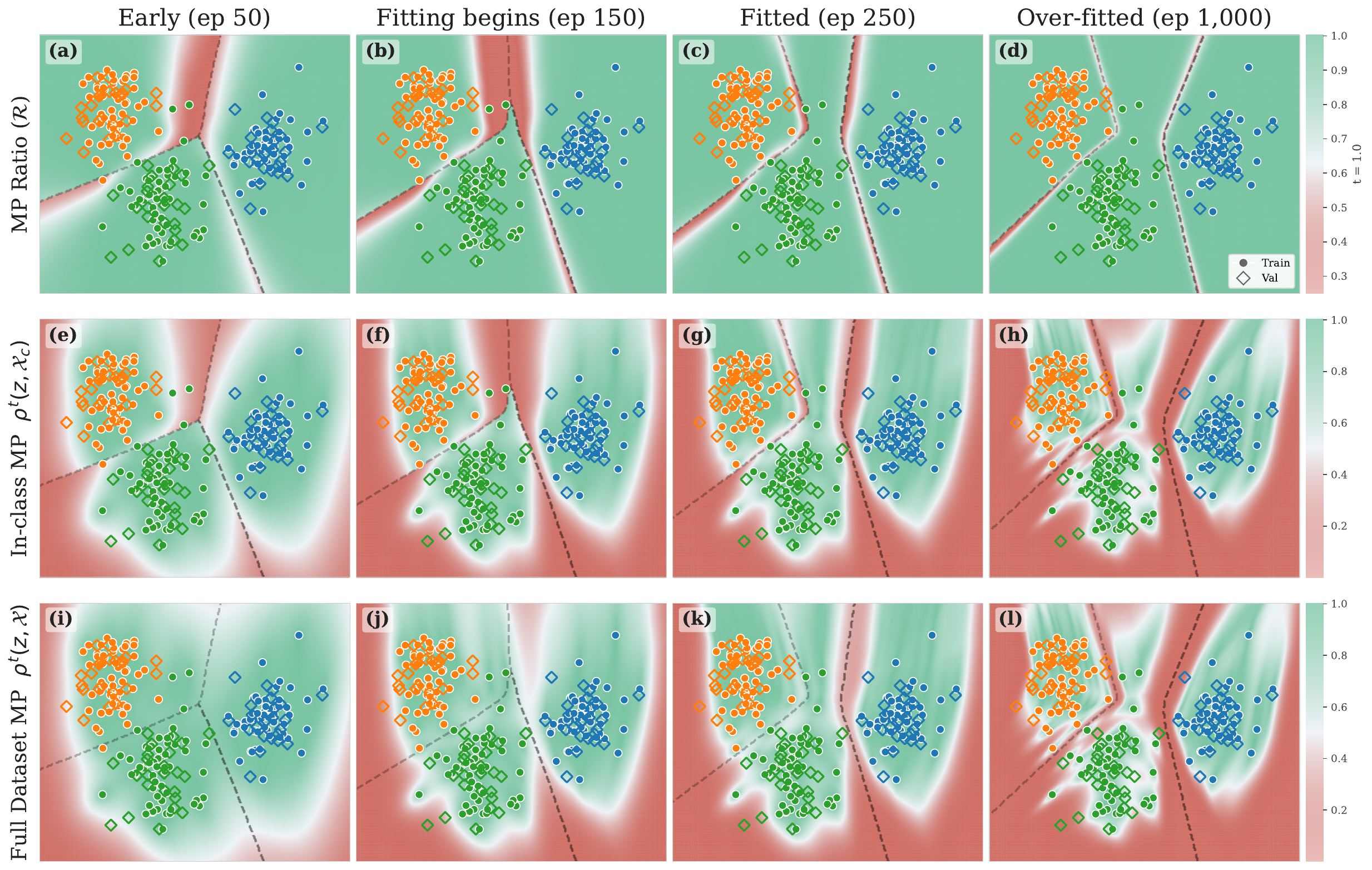}
  \end{center}
  \caption{
  The figure shows the role of in-class and the entire dataset magnitude potential ($t = 1.0$) in the ratio.
  Magnitude potential decomposition on the 2D toy (Gaussian) dataset at four key epochs (columns), aligned with the epochs marked in the dynamics figure. 
  Row 1 is the MP ratio, the magnitude potential within the class to the entire dataset in the logit-space.
  Row 2 is the in-class potential for each point in the boundary region of a corresponding class. 
  Row 3 is the entire dataset potential. 
  Filled circles are training points; open diamonds are validation points; white contours show decision boundaries.
  }
  \label{fig:mp_heatmap}
\end{figure}

\newpage

\section{Grokking in Modular Addition -- MLP with MSE}
\label{app:modular_mse}
We provide further information on this experiment.
\begin{figure}[htbp]
  \begin{center}
    \includegraphics[width=1.0\linewidth]{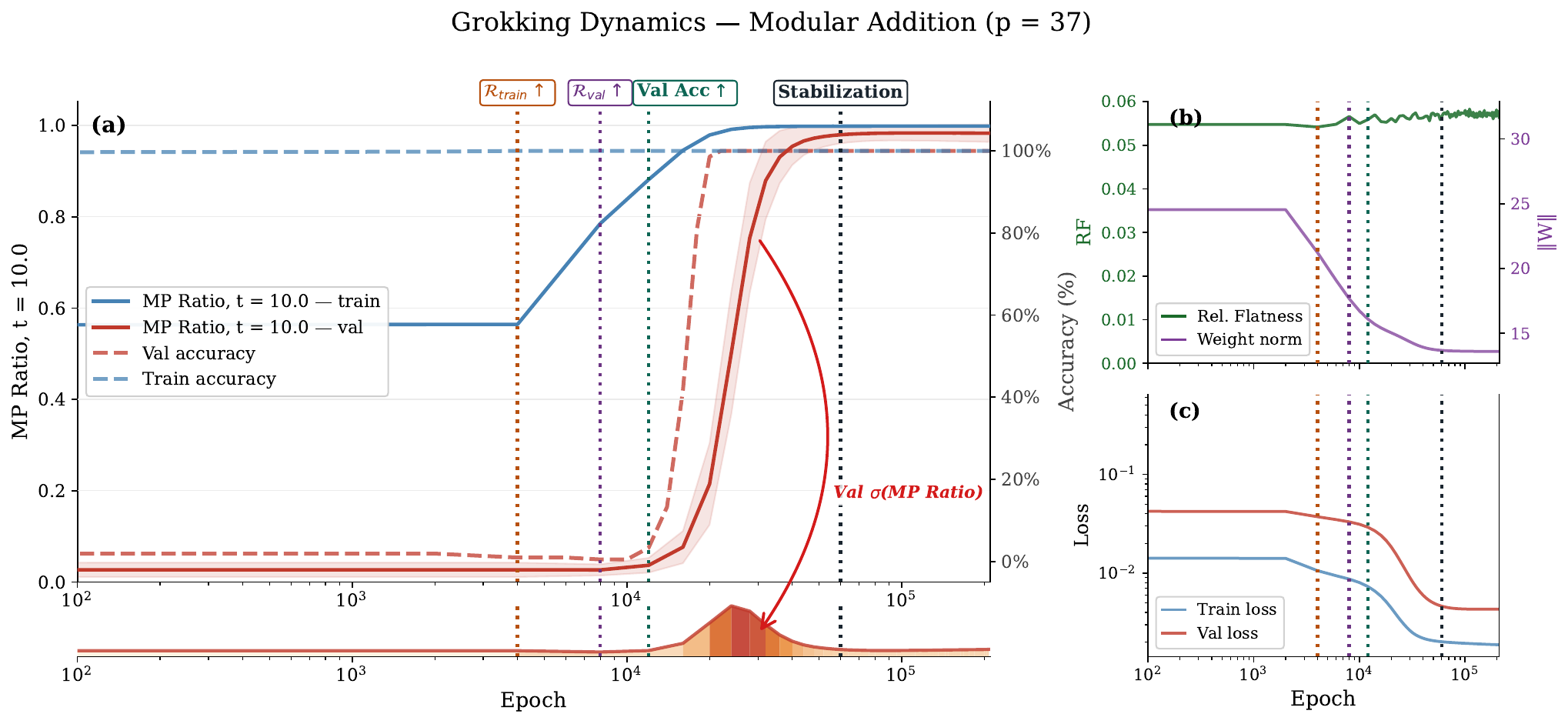}
  \end{center}
  \caption{Grokking dynamics — modular addition ($ \pmod{37}$). \textbf{(a)}~MP ratio, $t=10$) for the training (blue) and validation (red) sets, with accuracy curves overlaid (dashed). The shaded band denotes $\pm 1$ standard deviation of the validation MP ratio and the narrow strip below plots this standard deviation directly.
  \textbf{(b)}~Relative flatness $\kappa$ \citep{petzka2021relative} (green) and weight norm $\|W\|$ (purple).
  \textbf{(c)}~Train and validation loss (log scale). Vertical markers indicate: onset of train MP ratio rise $\mathcal{R}_{\mathrm{tr}}{\uparrow}$ (orange), onset of validation MP ratio rise $\mathcal{R}_{\mathrm{val}}{\uparrow}$ (violet), validation accuracy increasing (teal, $\uparrow$), validation accuracy reaching near-perfect generalization (green, $\sim$), and a stabilization point (black). All panels share the same log-scaled epoch axis, markers are aligned across panels.}
  \label{fig:modular_grokking}
\end{figure}

\section{Grokking in Modular Addition -- Transformer with Cross-Entropy, Sparse Training (\citet{nanda2023progress})}
\label{app:modular_transformer}
\textbf{Dataset:}
Modular addition over $\mathbb{Z}_{113}$: all $p^2 = 12{,}769$ ordered pairs $(a,b)$ with $0 \le a,b < p=113$, labelled by $(a+b)\bmod p$. Each input is the token-ID sequence $[a,\, b,\, {=}]$ (integer IDs), following
\citet{nanda2023progress}. The data is split 30/70 into train/validation (sparse-data regime).

\textbf{Model:}
A 1-layer Transformer with $d_{\text{model}}=128$, 4 attention heads, MLP hidden dimension $d_{\text{mlp}}=512$, ReLU activations, and no LayerNorm, with learned token and positional embeddings over a vocabulary of size $p+1$. Trained with AdamW ($\eta=10^{-3}$, $\lambda=1.0$, full-batch) using cross-entropy loss for 40\,000 epochs.

\textbf{Computation:}
Training completed in $3.7$ minutes on an NVIDIA L4 GPU (23 GB), and the entire experiment  MP ratio trajectories are computed at $51$ snapshots across $9$ temperature values $t \in {0.01, 0.05, 0.2, 0.5, 1.0, 3.0, 5.0, 10.0, 50.0}$, taking $6.7$ hours. Meaning that on average it takes about $44$ minutes to compute the MP ratios through the entire training process with $51$ snapshots, and each snapshot's computation is, on average, less than a minute.

\begin{figure}[htbp]
     \centering
     \begin{subfigure}[t]{0.47\textwidth}
         \centering
         \includegraphics[width=\linewidth]{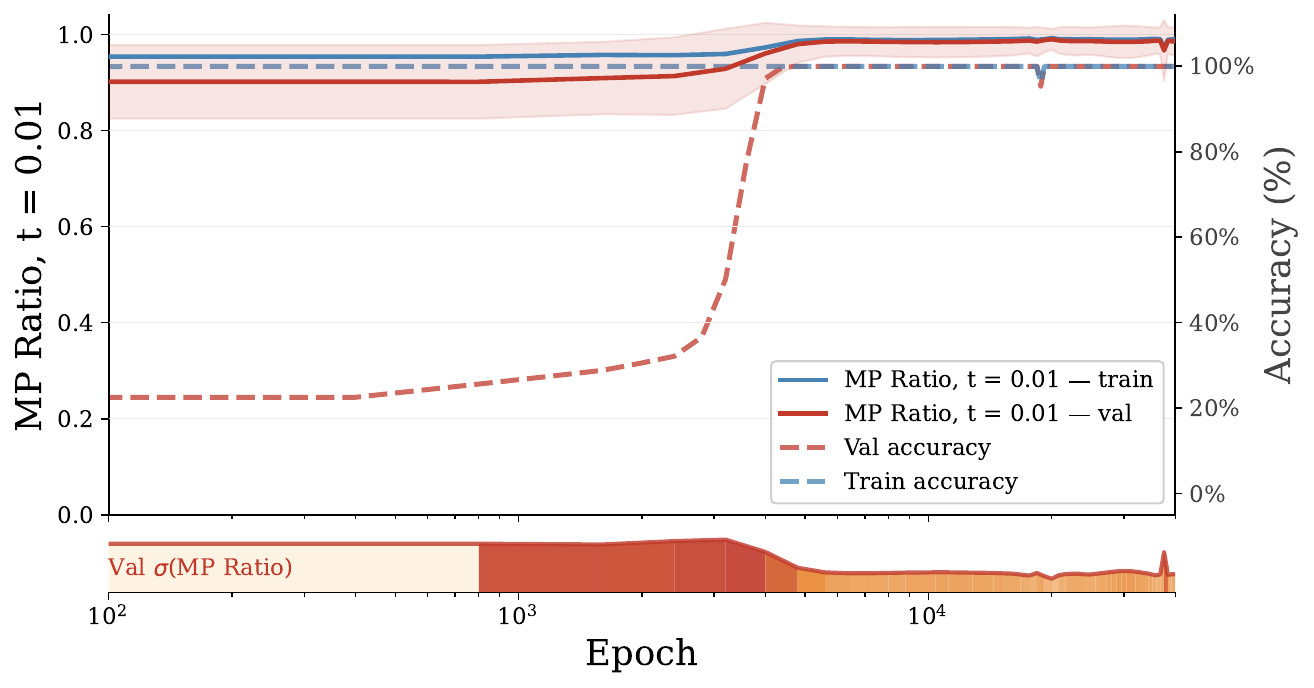}
         \caption{MP ratio vs accuracy functions.}
         \label{subfig:mp_ratio}
     \end{subfigure}
     \hfill
     \begin{subfigure}[t]{0.49\textwidth}
         \centering
         \includegraphics[width=\linewidth]{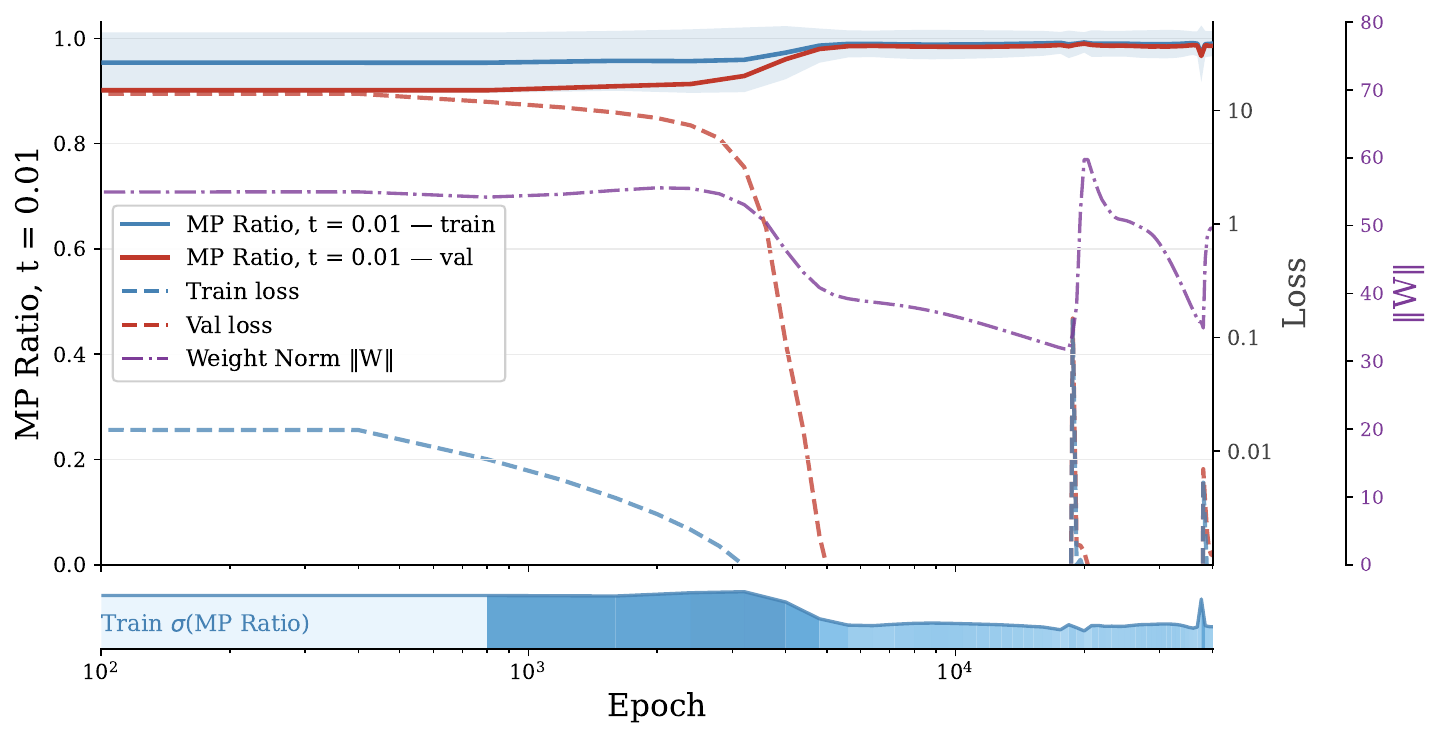} 
         \caption{MP ratio vs loss functions and weight norm.}
         \label{subfig:complexity}
     \end{subfigure}
  \caption{ \textbf{(a)}~MP ratio, $t=0.01$ for the training (blue) and validation (red) sets, with accuracy curves overlaid (dashed). The shaded band denotes the standard deviation of the validation MP ratio, and the narrow strip below plots this standard deviation directly.
  \textbf{(b)}~Train/validation loss(dashed), and weight norm $\|W\|$ (purple).
  }
  \label{fig:modular_b}
\end{figure}

\newpage
\section{Robustness to Neural Collapse -- CIFAR-10}
\begin{figure}[h]
    \centering
    \begin{subfigure}[b]{0.34\textwidth}
        \includegraphics[width=\textwidth]{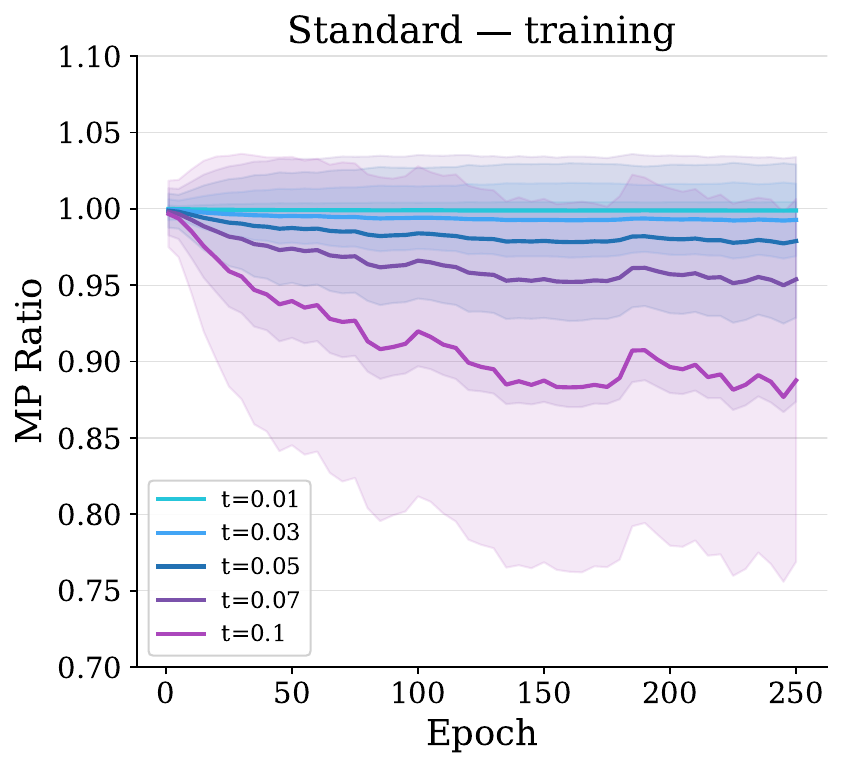}
        \caption{Standard — training}
    \end{subfigure}
    \hspace{40pt}
    \begin{subfigure}[b]{0.34\textwidth}
        \includegraphics[width=\textwidth]{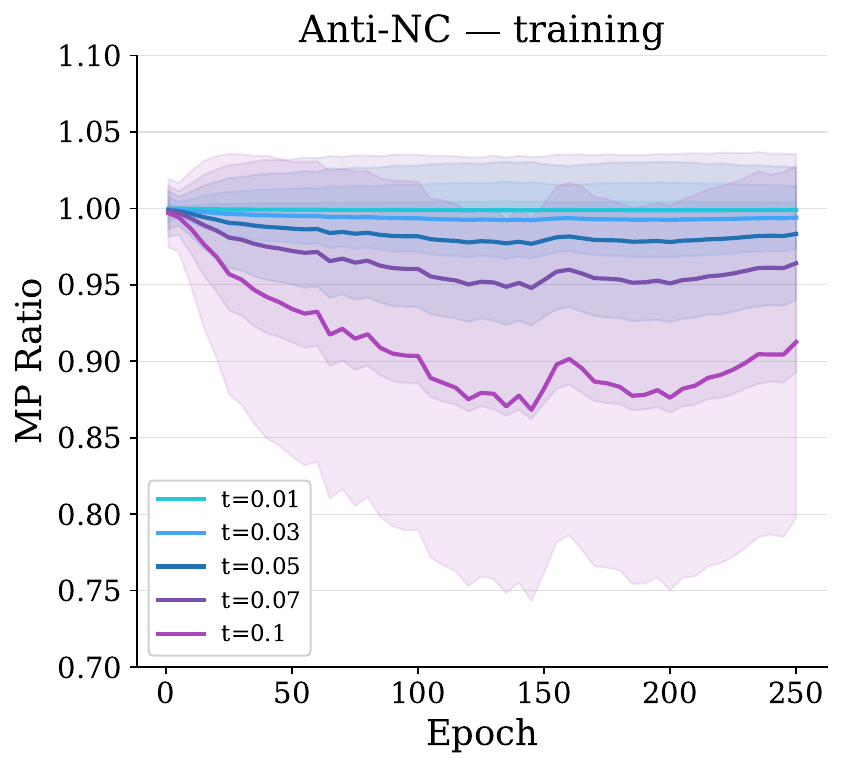}
        \caption{Anti-NC — training}
    \end{subfigure}

    \vspace{4pt}

    \begin{subfigure}[b]{0.34\textwidth}
        \includegraphics[width=\textwidth]{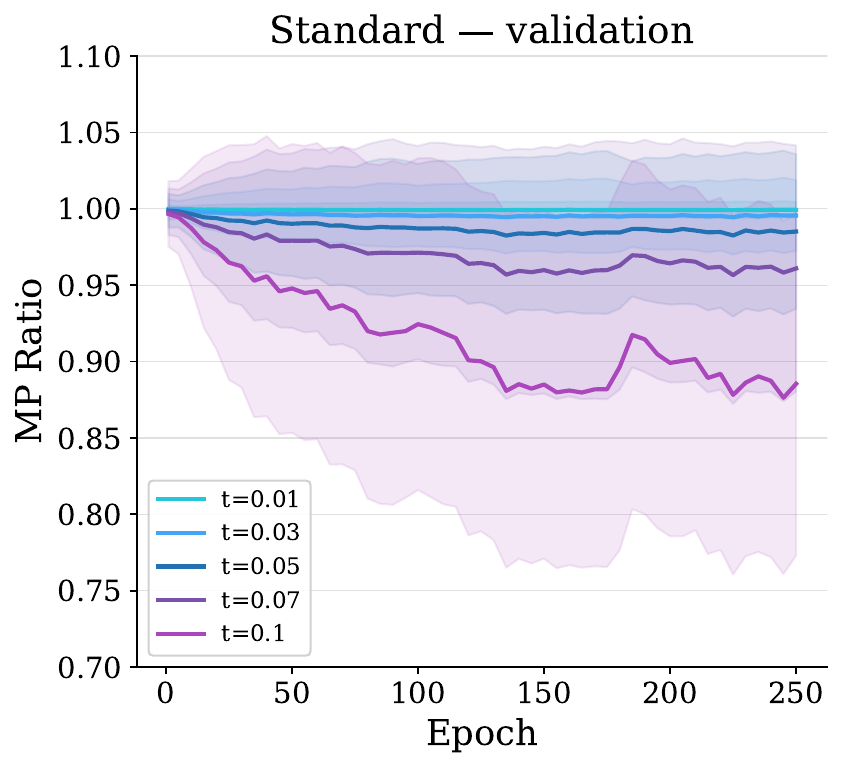}
        \caption{Standard — validation}
    \end{subfigure}
    \hspace{40pt}
    \begin{subfigure}[b]{0.34\textwidth}
        \includegraphics[width=\textwidth]{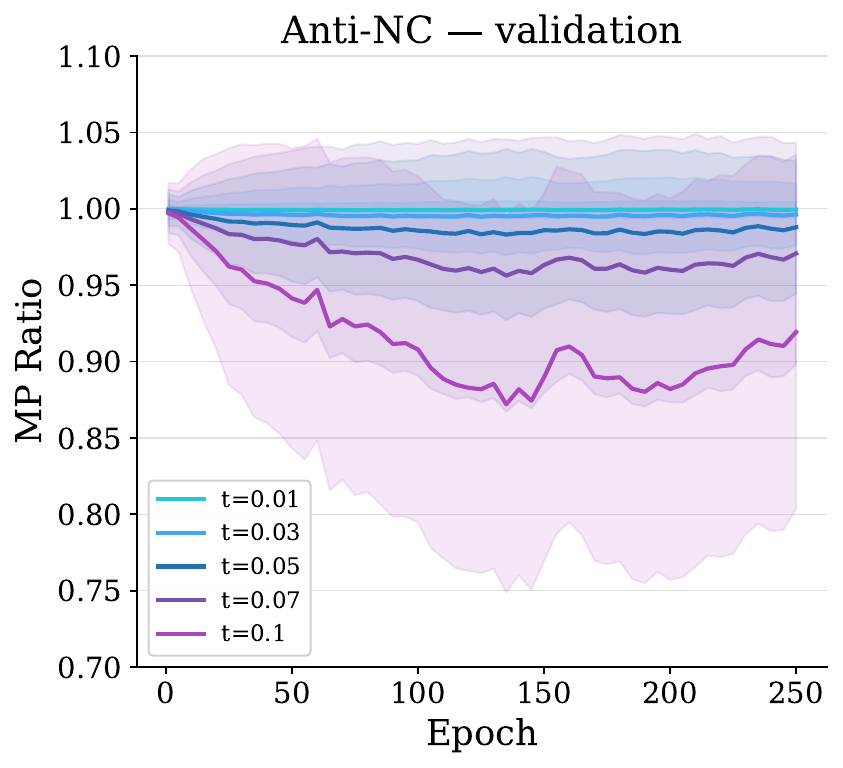}
        \caption{Anti-NC — validation}
    \end{subfigure}

    \caption{MP ratio over training epochs for multiple values of $t$ (mean $\pm$ std).
    Top row: training set. Bottom row: validation set.
    Left column: standard cross-entropy model. Right column: Anti-NC regularized model.
    The trajectory remains qualitatively similar, indicating that MP ratio is robust to both the choice of scale parameter and whether neural collapse is encouraged or suppressed.}
    \label{fig:multi_t}
\end{figure}

\newpage
\section{Per Sample Analysis: Memorization Detection on CIFAR-100}\label{app:per_sample_mem}
\begin{figure}[h]
        \centering
    \begin{subfigure}[b]{0.32\textwidth}
        \includegraphics[width=\textwidth]{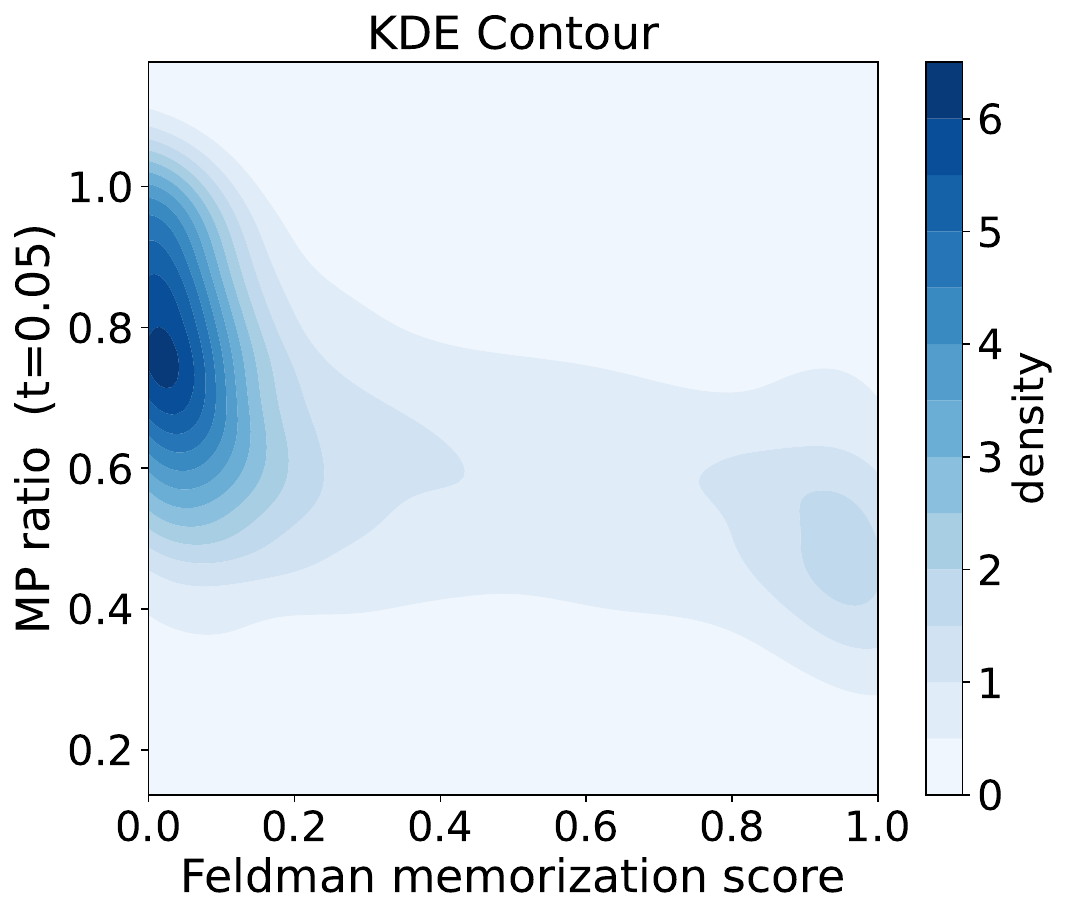}
        \caption{Single, $t=0.05$}
    \end{subfigure}
    \hfill
    \begin{subfigure}[b]{0.32\textwidth}
        \includegraphics[width=\textwidth]{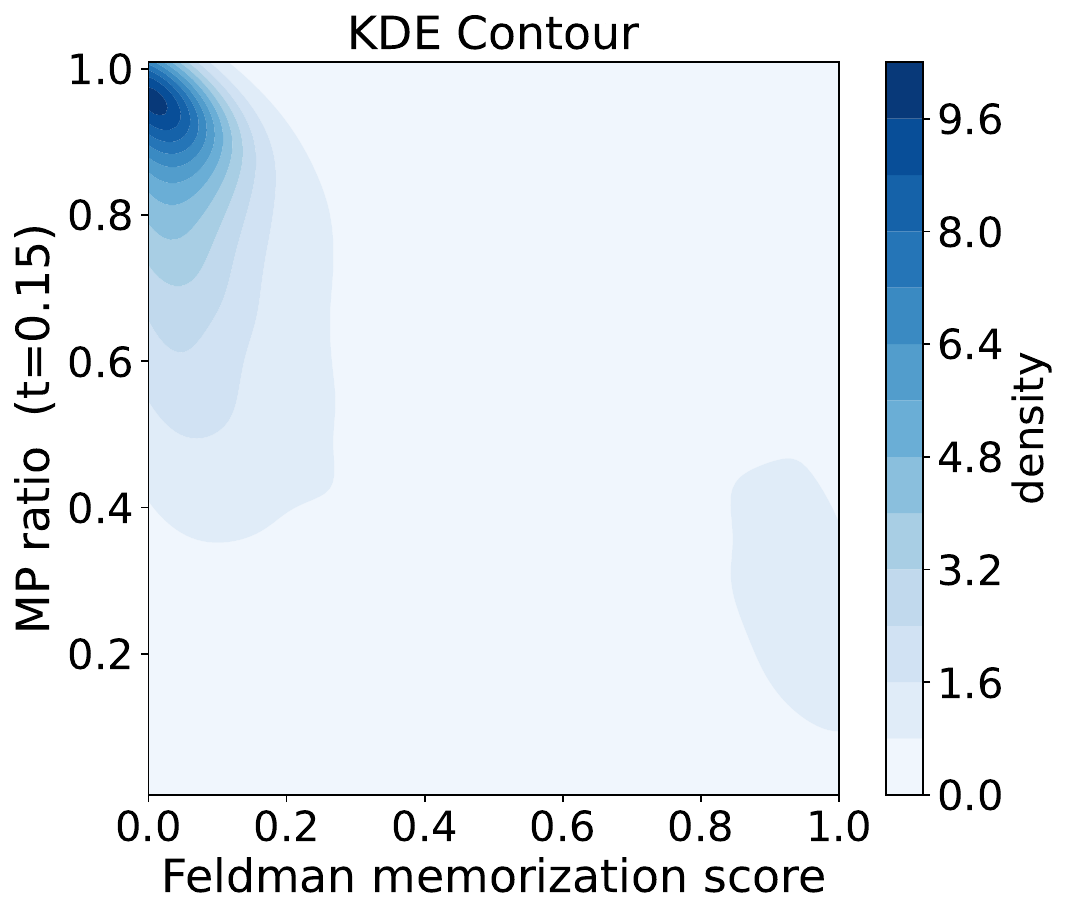}
        \caption{Single, $t=0.15$}
    \end{subfigure}
    \hfill
    \begin{subfigure}[b]{0.32\textwidth}
        \includegraphics[width=\textwidth]{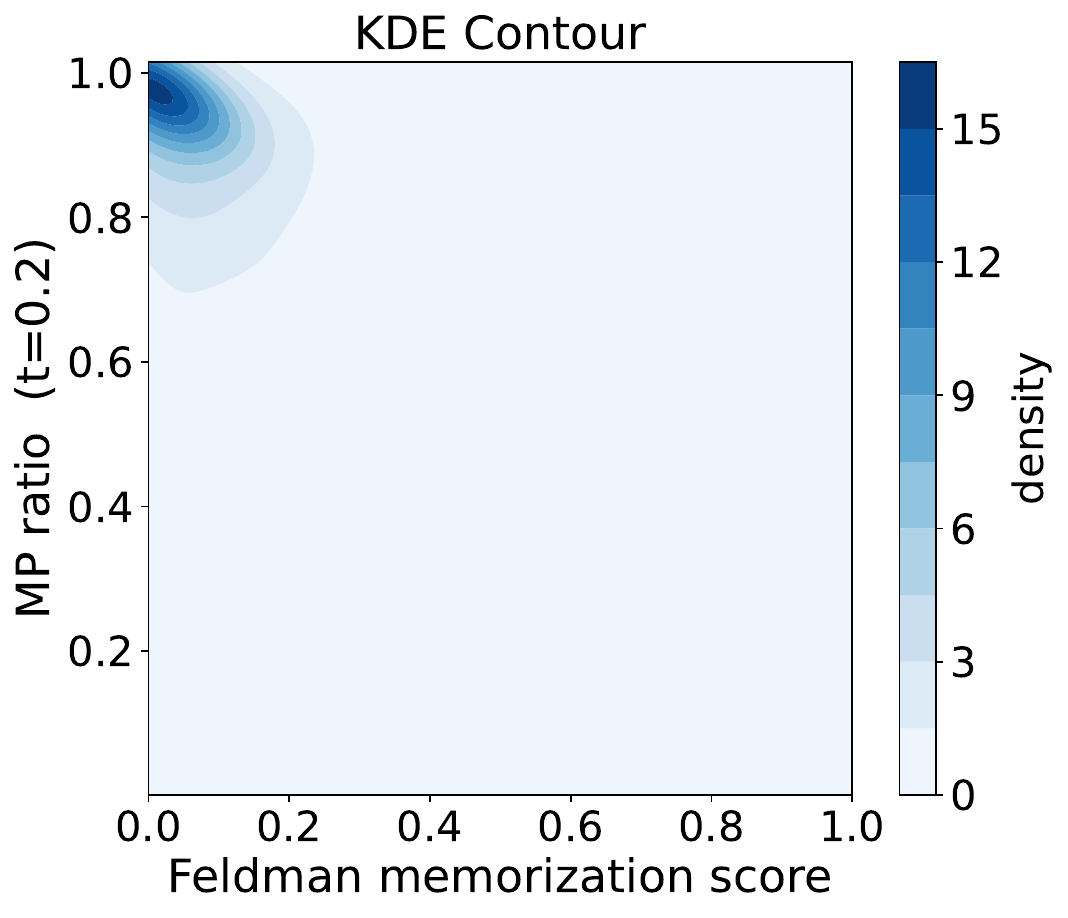}
        \caption{Single, $t=0.20$}
    \end{subfigure}

    \vspace{6pt}

    \begin{subfigure}[b]{0.32\textwidth}
        \includegraphics[width=\textwidth]{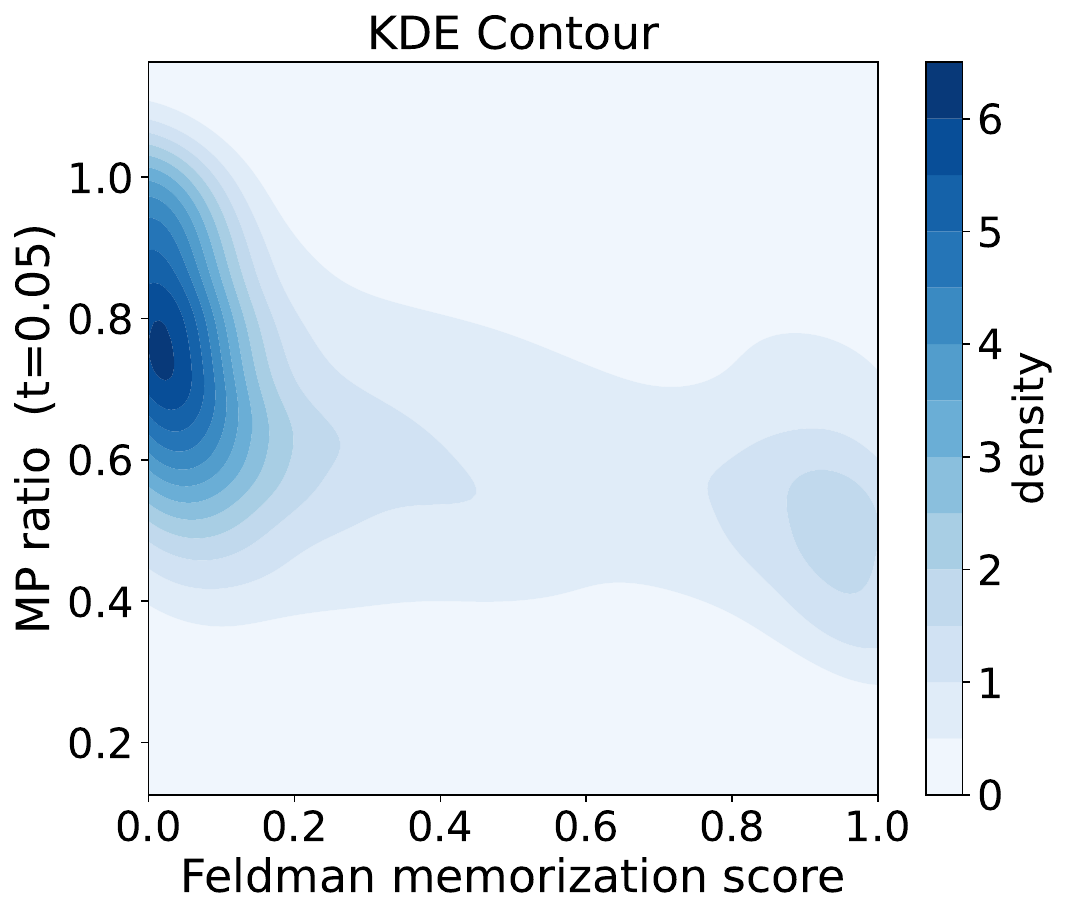}
        \caption{10-seed avg, $t=0.05$}
    \end{subfigure}
    \hfill
    \begin{subfigure}[b]{0.32\textwidth}
        \includegraphics[width=\textwidth]{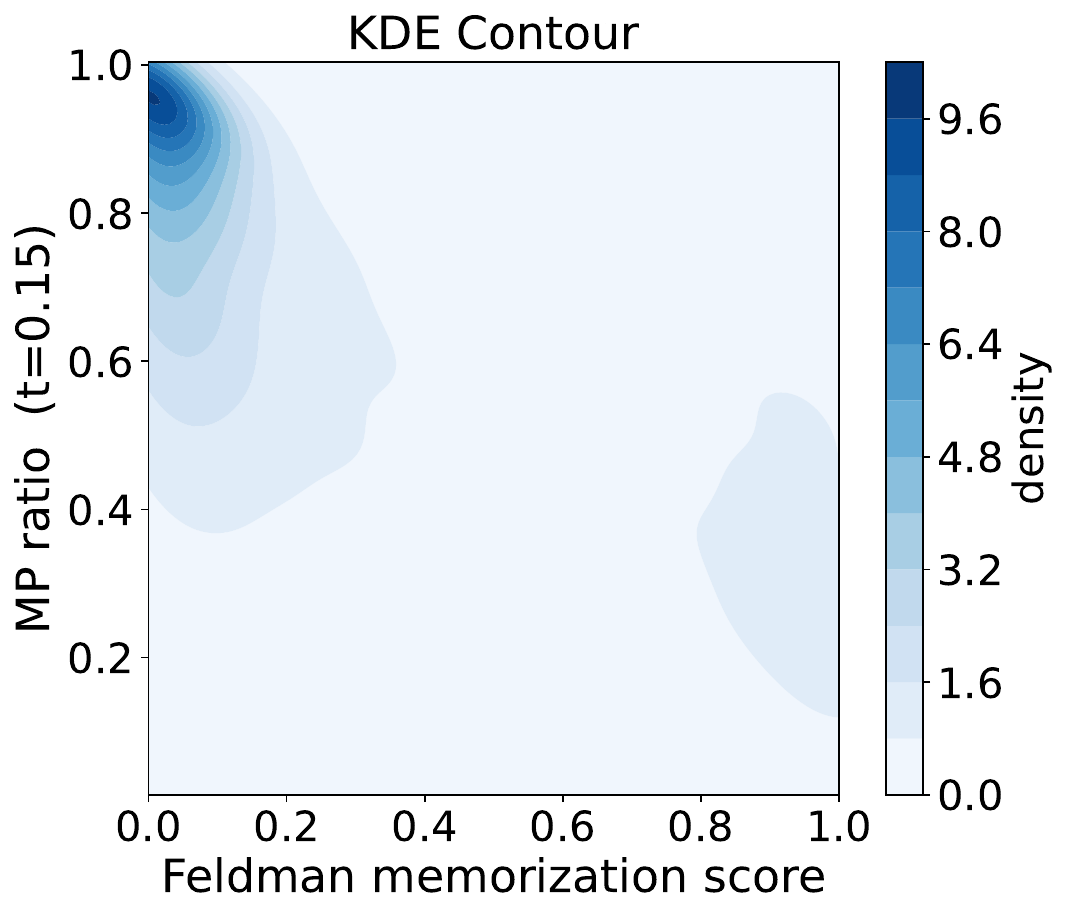}
        \caption{10-seed avg, $t=0.15$}
    \end{subfigure}
    \hfill
    \begin{subfigure}[b]{0.32\textwidth}
        \includegraphics[width=\textwidth]{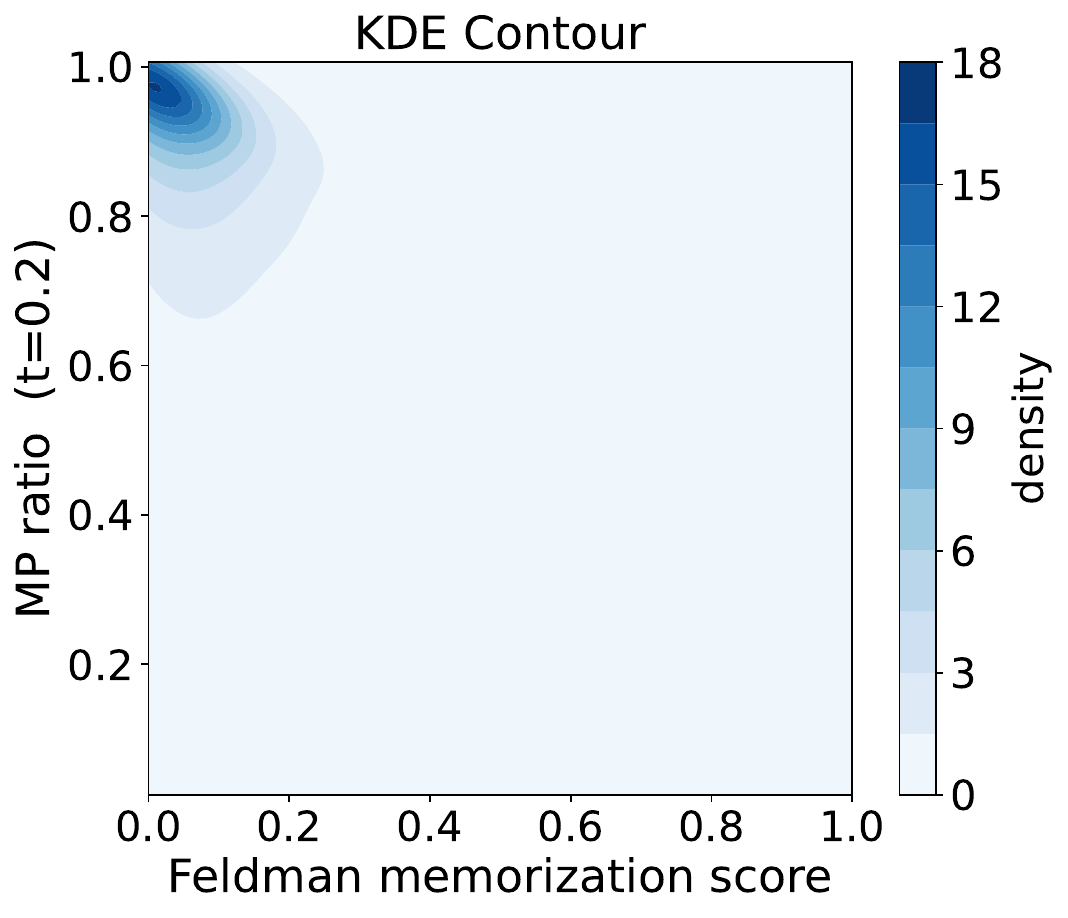}
        \caption{10-seed avg, $t=0.20$}
    \end{subfigure}

    \caption{Joint density of MP ratio versus Feldman memorization score on CIFAR-100 (KDE filled contours), with subsampling uniformly $200$ points per class). In-cluster points with low memorisation scores form a large cluster with higher MP ratio, while memorised points with low MP ratio iis found more toward bottom right. Top row: single random subsample. Bottom row: average over 10 runs of independent subsamplings. %
    The three columns show how the sensitivity of the MP ratio to local geometry varies with the scaling parameter~$t$. For larger $t$ up to 0.20, the differences in MP ratio between memorised and non memorised points are larger. 
    }
    \label{fig:feldman_kde}
\end{figure}

\newpage
Per-sample memorization values exhibit greater MP ratio variability, rather than the mean, as we demonstrate in Figure~\ref{fig:feldman_spaghetti} and Figure~\ref{fig:feldman_scatter_errbar}.

\begin{figure}[h]
    \centering
    \includegraphics[width=0.32\linewidth]{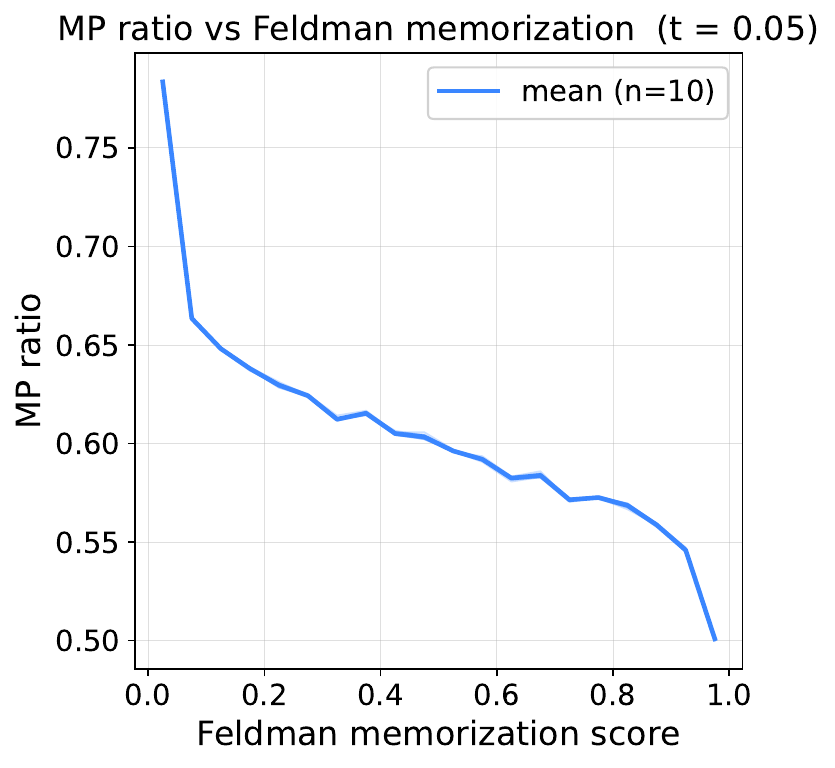}
    \hfill
    \includegraphics[width=0.32\linewidth]{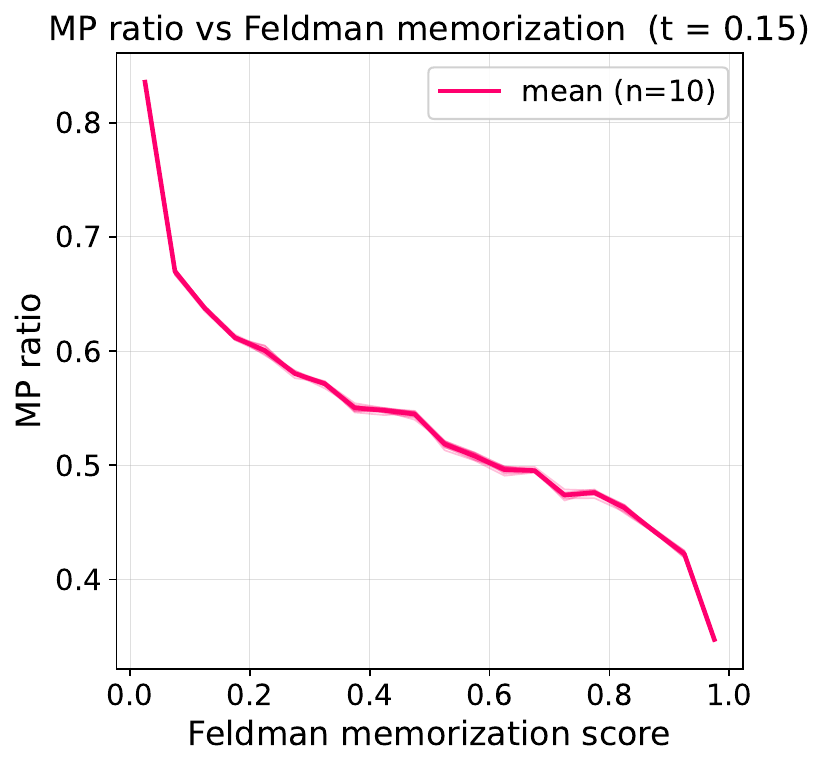}
    \hfill
    \includegraphics[width=0.32\linewidth]{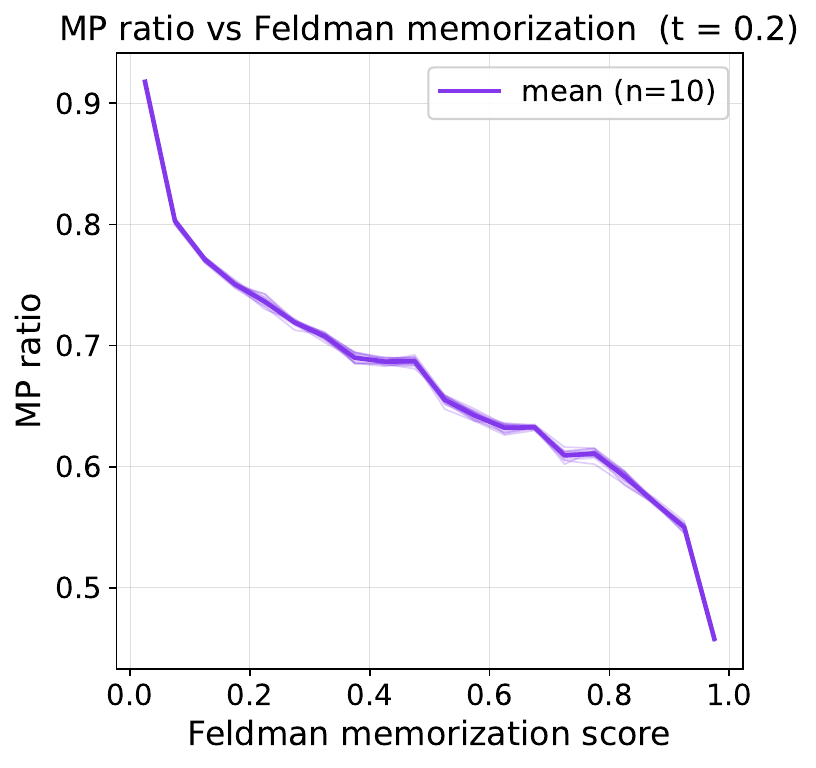}
    \caption{Binned mean MP ratio vs.\ Feldman memorization score across 10 random subsampling seeds ($t \in \{0.05, 0.15, 0.20\}$). Each thin line is one seed; the bold curve is the mean; the shaded band is $\pm$1 std. The tight bundle confirms that MP ratio is stable across different training subsets.}
    \label{fig:feldman_spaghetti}
\end{figure}

\begin{figure}[h]
    \centering
    \includegraphics[width=0.32\linewidth]{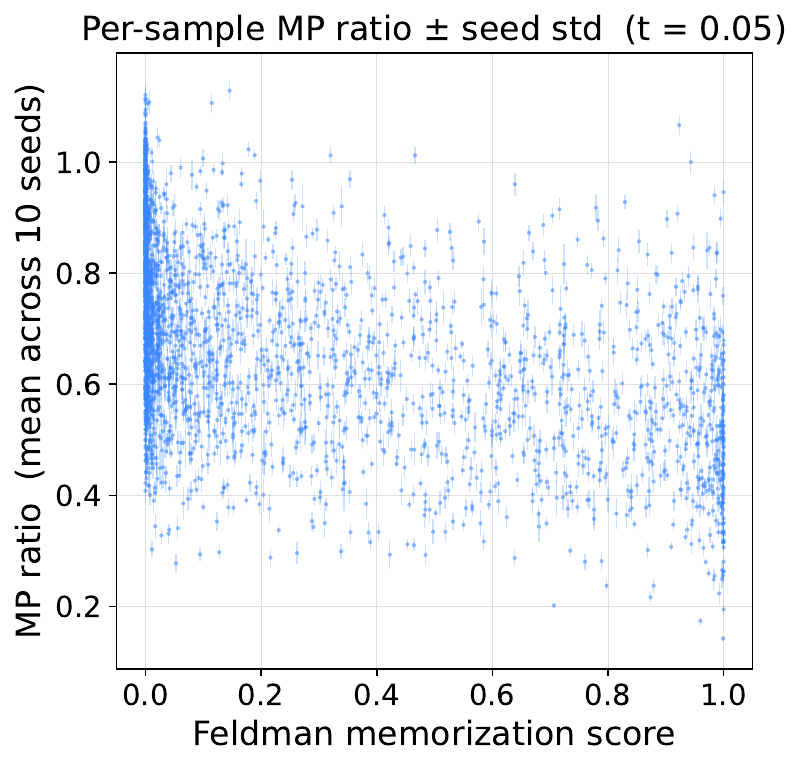}
    \hfill
    \includegraphics[width=0.32\linewidth]{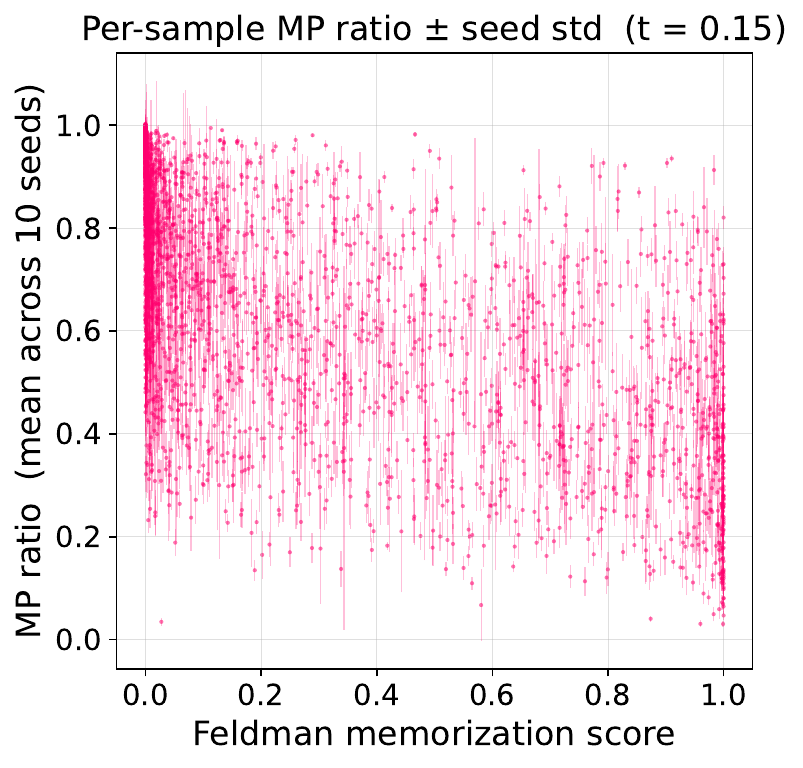}
    \hfill
    \includegraphics[width=0.32\linewidth]{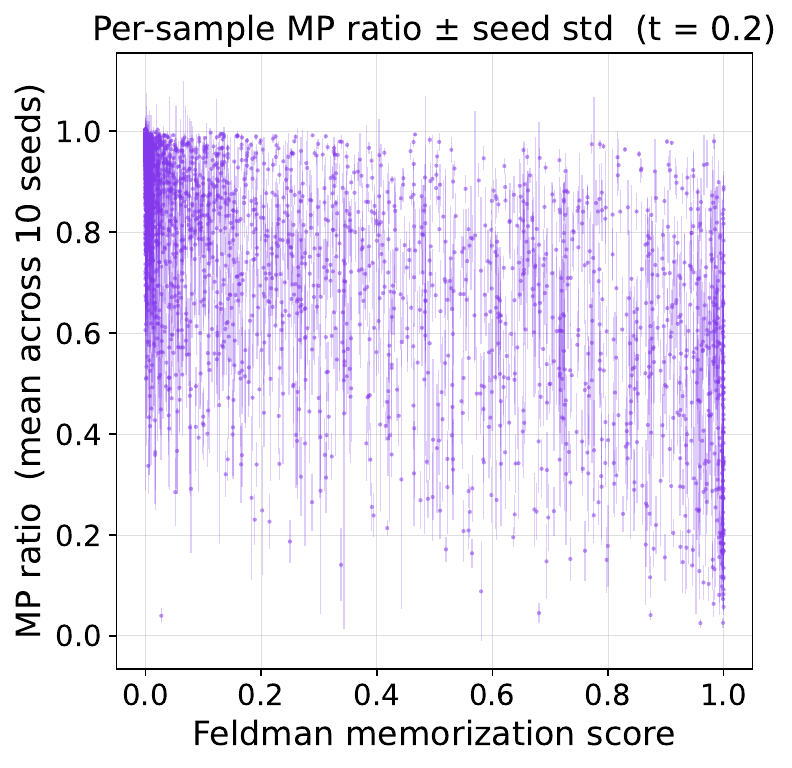}
    \caption{Per-sample MP ratio (mean across 10 seeds, $n=3000$ subsampled points) vs.\ Feldman memorization score. Vertical bars show per-sample standard deviation across seeds.}
    \label{fig:feldman_scatter_errbar}
\end{figure}

\newpage

\begin{figure}[h]
    \centering
    \includegraphics[width=\textwidth]{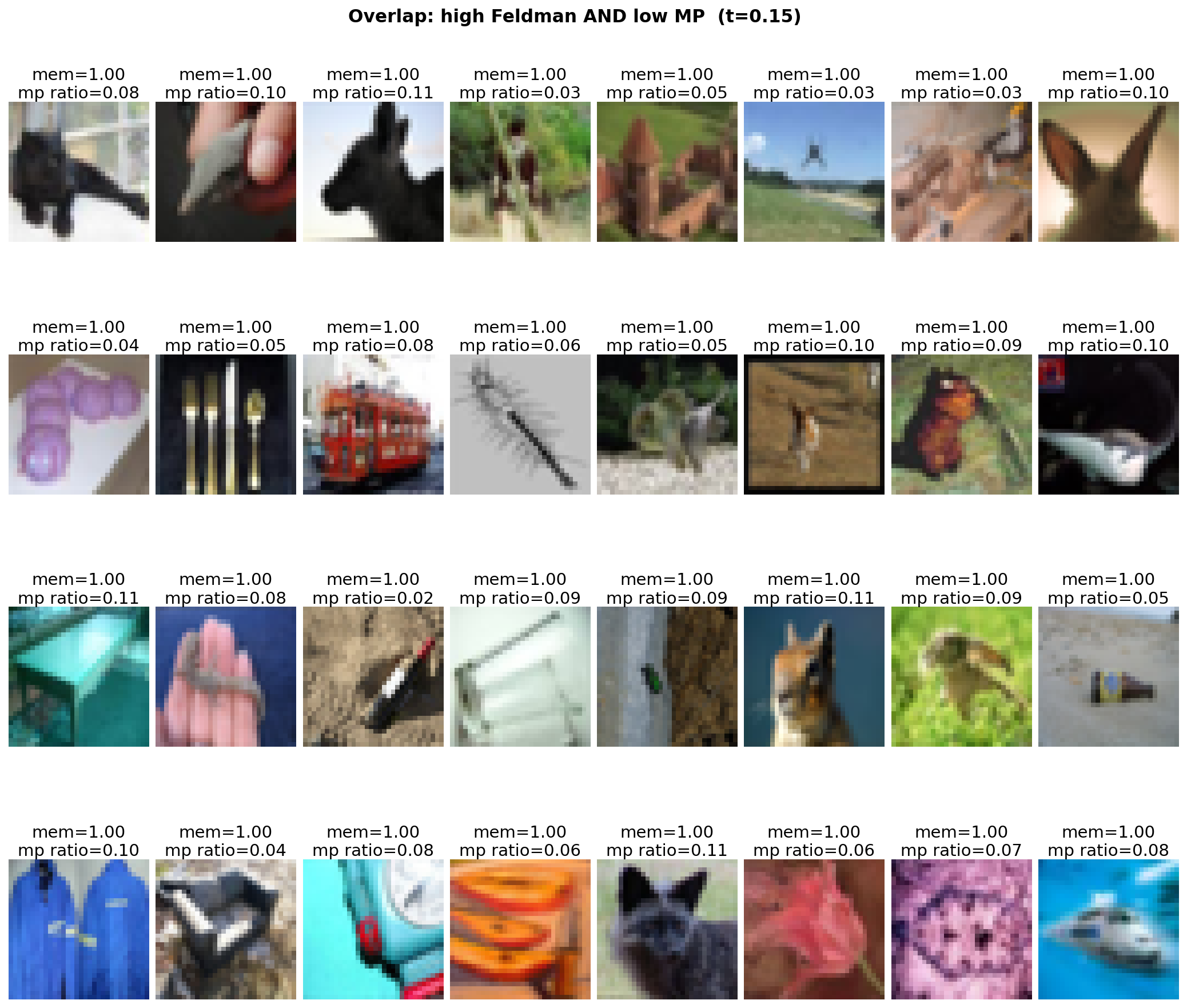}
    \caption{Training samples identified as highly memorized by \emph{both} methods:
    Feldman memorization score $= 1.00$ and MP ratio $< 0.12$ ($t{=}0.15$,
    10-seed average).
    Every sample shown has the maximum Feldman score, indicating it is individually memorized by the model, while its near-zero MP ratio reveals that its logit neighbourhood is geometrically inconsistent with the rest of its class. 
    These two signals are derived from entirely independent and different procedures yet they both find the same subset, indicating that MP ratio captures a meaningful notion of sample-level memorization.}
    \label{fig:feldman_overlap}
\end{figure}

\newpage

\begin{figure}[h]
    \centering
    \includegraphics[width=\textwidth]{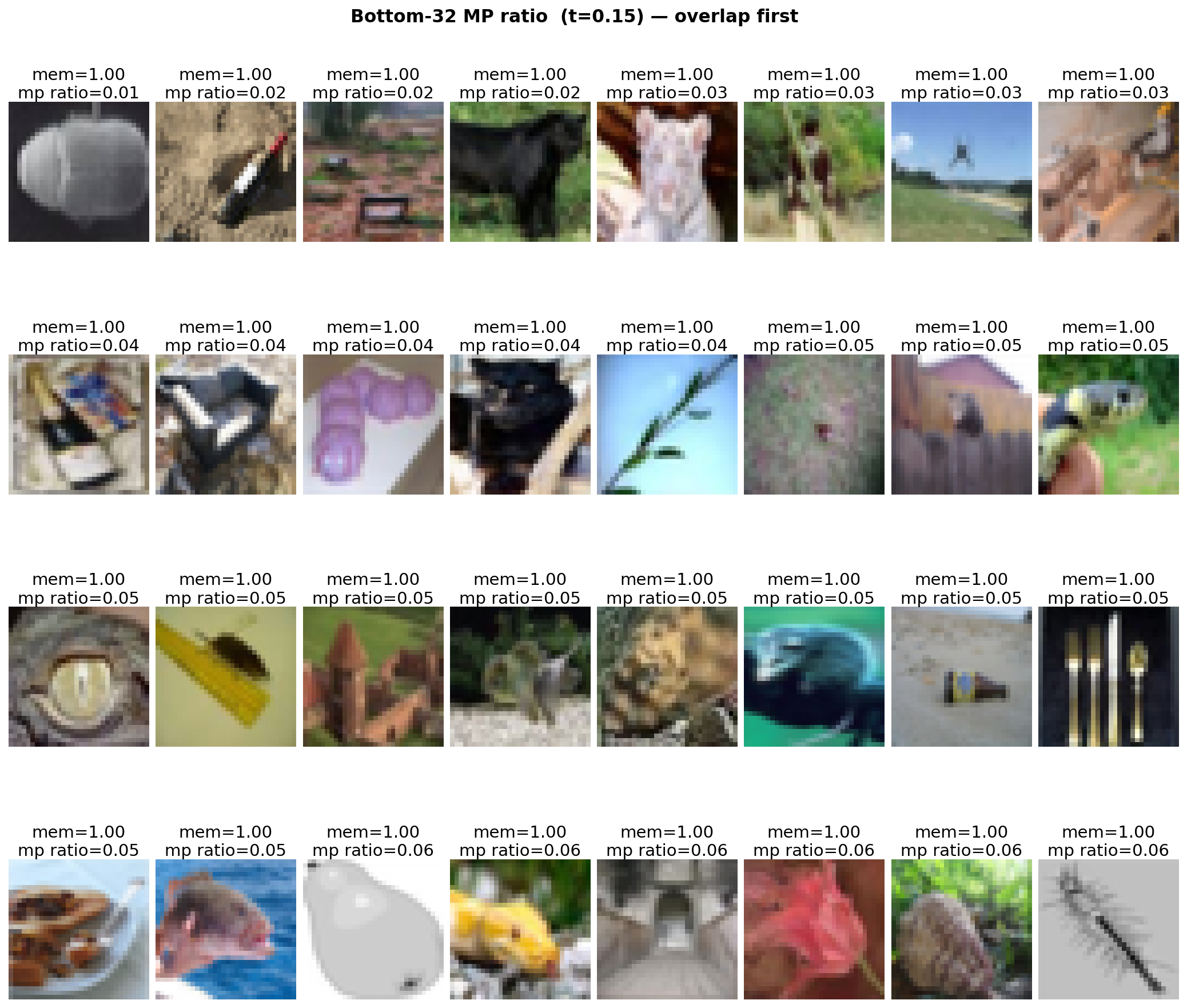}
    \caption{The 32 training samples with the lowest MP ratio ($t{=}0.15$,
    10-seed average), sorted ascending and with overlap samples shown first.
    Nearly all carry Feldman memorization score $= 1.00$, indicating that
    the samples the magnitude potential identifies as geometrically atypical
    are similar to the ones that require individual memorization
    to be correctly classified.}
    \label{fig:feldman_low_mp}
\end{figure}

\newpage

\section*{NeurIPS Paper Checklist}

\begin{enumerate}

\item {\bf Claims}
    \item[] Question: Do the main claims made in the abstract and introduction accurately reflect the paper's contributions and scope?
    \item[] Answer: \answerYes{} %
    \item[] Justification: The abstract and introduction clearly state that we introduce Magnitude Potential (MP) and Magnitude Potential ratio (MP ratio), and use it to study neural network training -- memorization, grokking, decision boundaries and relation to neural collapse. 

\item {\bf Limitations}
    \item[] Question: Does the paper discuss the limitations of the work performed by the authors?
    \item[] Answer: \answerYes{} %
    \item[] Justification: We address limitations of the proposed method in a separate section of \textit{Limitations} (Section~\ref{sec:limitations}).

\item {\bf Theory assumptions and proofs}
    \item[] Question: For each theoretical result, does the paper provide the full set of assumptions and a complete (and correct) proof?
    \item[] Answer: \answerYes{} %
    \item[] Justification: We clarify the assumption of our framework in the beginning of Section~\ref{sec:mp}. Further details and common cases for the assumptions to be held are discussed in Appendix~\ref{app:assumptions}.
    The full proofs are provided in the Appendix~\ref{app:proofs}.

    \item {\bf Experimental result reproducibility}
    \item[] Question: Does the paper fully disclose all the information needed to reproduce the main experimental results of the paper to the extent that it affects the main claims and/or conclusions of the paper (regardless of whether the code and data are provided or not)?
    \item[] Answer: \answerYes{} %
    \item[] Justification: All required information including main hyperparameters, architectures, optimizers, data splits for each experiment is reported in the corresponding section. An anonymized implementation of the Magnitude Potential ratio, documented via a \texttt{README} covering the required reproduction instructions is provided in the supplemental material. 

\item {\bf Open access to data and code}
    \item[] Question: Does the paper provide open access to the data and code, with sufficient instructions to faithfully reproduce the main experimental results, as described in supplemental material?
    \item[] Answer: \answerYes{} %
    \item[] Justification: An anonymized code repository is included in the supplemental material. All datasets used (CIFAR-10, CIFAR-100, modular-arithmetic splits, Feldman memorization scores) are publicly available and download instructions are included in the repository README.

\item {\bf Experimental setting/details}
    \item[] Question: Does the paper specify all the training and test details (e.g., data splits, hyperparameters, how they were chosen, type of optimizer) necessary to understand the results?
    \item[] Answer: \answerYes{} %
    \item[] Justification: At each section with empirical results, all required information and training configurations are reported for both \textbf{Data} and \textbf{Model}. Further details are provided in the Appendix~\ref{app:experimental_setup}. Code is provided in the supplemental material. 

\item {\bf Experiment statistical significance}
    \item[] Question: Does the paper report error bars suitably and correctly defined or other appropriate information about the statistical significance of the experiments?
    \item[] Answer: \answerYes{} %
    \item[] Justification:  Most experiments are deterministic given a fixed dataset and training run, therefore error bars are not required. However, we provide results across multiple values of the scaling parameter $t$ in the Appendix corresponding to each experiment for completeness.
    For the Feldman memorization comparison, which involves random subsampling of training examples, we report per-sample standard deviation across 10 independent seeds in Figure~\ref{fig:feldman_scatter_errbar} and binned mean $\pm$1\,std curves in Figure~\ref{fig:feldman_spaghetti} (Appendix~\ref{app:per_sample_mem}).

\item {\bf Experiments compute resources}
    \item[] Question: For each experiment, does the paper provide sufficient information on the computer resources (type of compute workers, memory, time of execution) needed to reproduce the experiments?
    \item[] Answer: \answerYes{}{} %
    \item[] Justification: GPU type, approximate running time, and other requirements are reported in the Appendix~\ref{app:experimental_setup}.
    
\item {\bf Code of ethics}
    \item[] Question: Does the research conducted in the paper conform, in every respect, with the NeurIPS Code of Ethics \url{https://neurips.cc/public/EthicsGuidelines}?
    \item[] Answer: \answerYes{} %
    \item[] Justification: The research involves no human subjects, sensitive data, or dual-use risks and it fully conforms with the NeurIPS Code of Ethics.

\item {\bf Broader impacts}
    \item[] Question: Does the paper discuss both potential positive societal impacts and negative societal impacts of the work performed?
    \item[] Answer: \answerNA{} %
    \item[] Justification: This work is foundational research in nature. It does not involve generative models, sensitive personal data, or any direct path to negative applications.
    
\item {\bf Safeguards}
    \item[] Question: Does the paper describe safeguards that have been put in place for responsible release of data or models that have a high risk for misuse (e.g., pre-trained language models, image generators, or scraped datasets)?
    \item[] Answer: \answerNA{} %
    \item[] Justification: No high-risk models, generative systems, or scraped datasets are released. The method is an analytical measure applied to existing, publicly available benchmarks (CIFAR-10 an d CIFAR-100) and 2D Gaussian datasets.

\item {\bf Licenses for existing assets}
    \item[] Question: Are the creators or original owners of assets (e.g., code, data, models), used in the paper, properly credited and are the license and terms of use explicitly mentioned and properly respected?
    \item[] Answer: \answerYes{} %
    \item[] Justification: Datasets, CIFAR-10 and CIFAR-100 are released under the MIT licence. The Feldman memorization scores~\citep{feldman2020neural} are released under CC-BY~4.0 and are used solely for comparison. All of them are properly cited and further details on assets including common libraries are their licences are listed in \texttt{LICENSES.md} in the supplemental code repository. 

\item {\bf New assets}
    \item[] Question: Are new assets introduced in the paper well documented and is the documentation provided alongside the assets?
    \item[] Answer: \answerYes{} %
    \item[] Justification: We include an anonymized implementation of Magnitude Potential ratio as supplemental material, documented via a \texttt{README} covering the required reproduction instructions. No new datasets are introduced. The license has been omitted from the supplemental material to preserve anonymity, and the full MIT-licensed repository will be made public upon acceptance.

\item {\bf Crowdsourcing and research with human subjects}
    \item[] Question: For crowdsourcing experiments and research with human subjects, does the paper include the full text of instructions given to participants and screenshots, if applicable, as well as details about compensation (if any)? 
    \item[] Answer: \answerNA{} %
    \item[] Justification: The paper does not involve crowdsourcing nor research with human subjects.

\item {\bf Institutional review board (IRB) approvals or equivalent for research with human subjects}
    \item[] Question: Does the paper describe potential risks incurred by study participants, whether such risks were disclosed to the subjects, and whether Institutional Review Board (IRB) approvals (or an equivalent approval/review based on the requirements of your country or institution) were obtained?
    \item[] Answer: \answerNA{} %
    \item[] Justification: The paper does not involve crowdsourcing nor research with human subjects.

\item {\bf Declaration of LLM usage}
    \item[] Question: Does the paper describe the usage of LLMs if it is an important, original, or non-standard component of the core methods in this research? Note that if the LLM is used only for writing, editing, or formatting purposes and does \emph{not} impact the core methodology, scientific rigor, or originality of the research, declaration is not required.
    \item[] Answer: \answerNA{} %
    \item[] Justification: LLMs were used solely for writing and editing assistance and do not form any part of the core contributions of this work.

\end{enumerate}

\end{document}